\documentclass[letterpaper]{article} 
\usepackage{aaai2026}  
\usepackage{times}  
\usepackage{helvet}  
\usepackage{courier}  
\usepackage[hyphens]{url}  
\usepackage{graphicx} 
\usepackage{natbib}  
\usepackage{caption} 
\usepackage{algorithm}
\usepackage{algorithmic}
\usepackage{newfloat}
\usepackage{listings}

\usepackage{times}
\usepackage{soul}
\usepackage{url}
\usepackage[utf8]{inputenc}
\usepackage{caption}

\usepackage{algorithm}
\usepackage{algorithmic}
\usepackage[switch]{lineno}

\usepackage{graphicx}
\usepackage{amsmath}
\usepackage{amssymb}
\usepackage{booktabs}
\usepackage{enumitem}
\usepackage{colortbl}
\usepackage{color, xcolor}
\usepackage{natbib}  
\usepackage{caption} 

\usepackage{multirow}
\usepackage{makecell}
\usepackage{array}
\usepackage{arydshln}
\usepackage{adjustbox}
\usepackage{nicefrac}

\usepackage{subfigure}
\usepackage{caption}
\usepackage{subcaption}

\usepackage{multirow}
\usepackage{booktabs}
\usepackage{float} 

\usepackage{amsmath}   
\usepackage{amsthm}    
\theoremstyle{plain} 
\newtheorem{definition}{Definition}

\newtheorem{lemma}{Lemma}

\DeclareCaptionStyle{ruled}{labelfont=normalfont,labelsep=colon,strut=off} 
\floatstyle{ruled}
\newfloat{listing}{tb}{lst}{}
\floatname{listing}{Listing}
\title{Towards Robust Image Denoising with Scale Equivariance}
\author{
    Dawei Zhang, Xiaojie Guo\thanks{Corresponding author.}
}
\affiliations{
    Tianjin University

}

\usepackage{bibentry}

\begin{document}

\maketitle

\begin{abstract}

Despite notable advances in image denoising, existing models often struggle to generalize beyond in-distribution noise patterns, particularly when confronted with out-of-distribution (OOD) conditions characterized by spatially variant noise. This generalization gap remains a fundamental yet underexplored challenge.
In this work, we investigate \emph{scale equivariance} as a core inductive bias for improving OOD robustness. We argue that incorporating scale-equivariant structures enables models to better adapt from training on spatially uniform noise to inference on spatially non-uniform degradations.
Building on this insight, we propose a robust blind denoising framework equipped with two key components: a Heterogeneous Normalization Module (HNM) and an Interactive Gating Module (IGM). HNM stabilizes feature distributions and dynamically corrects features under varying noise intensities, while IGM facilitates effective information modulation via gated interactions between signal and feature paths.
Extensive evaluations demonstrate that our model consistently outperforms state-of-the-art methods on both synthetic and real-world benchmarks, especially under spatially heterogeneous noise. Code will be made publicly available.
\end{abstract}

\section{Introduction}
The field of image denoising has progressed significantly, evolving from traditional prior-based methods to deep learning frameworks that deliver state-of-the-art performance. However, these advances are largely confined to a supervised learning paradigm, which exposes a fundamental vulnerability, \textit{i.e.}, models trained on paired data with in-distribution (ID) noise often fail to generalize when confronted with out-of-distribution (OOD) noise patterns that are more complex and spatially varying. As illustrated in Fig.~\ref{fig:intro_show}, performance degrades sharply under such mismatched conditions. 
This generalization gap is commonly attributed to overfitting, \textit{i.e.}, networks tend to entangle signal representations with the noise characteristics seen during training. Although recent self-supervised and zero-shot methods offer improved flexibility, they typically underperform in terms of denoising quality and remain fragile under true OOD settings. As such, the field faces a critical dilemma: \emph{how to attain the high fidelity of supervised denoising without inheriting its brittleness}.

\begin{figure}[t]
\centering
\includegraphics[scale=0.15]{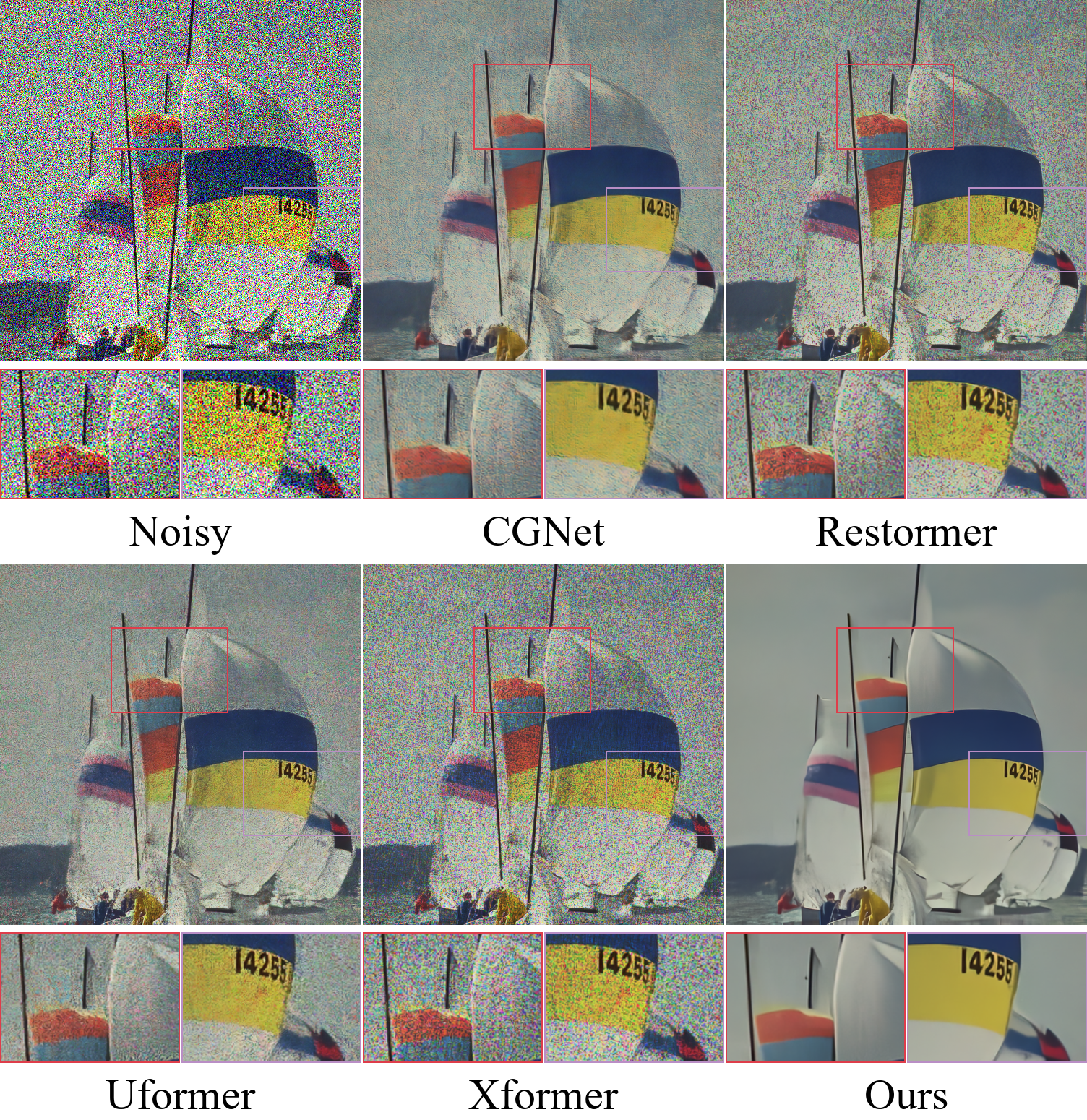}
\caption{Visual comparison on a sample with OOD noise. } 

\vspace{-17pt}
\label{fig:intro_show}
\end{figure}

To tackle this issue, existing approaches have explored several directions. One straightforward strategy is to diversify the training noise distribution, using synthesized complex noise such as mixtures~\cite{DBLP:conf/cvpr/GuoY0Z019, DBLP:journals/ijautcomp/ZhangLLCZTFTG23} or adversarial perturbations~\cite{DBLP:conf/cvpr/ChenG0MDWPZ23}. Another line of work focuses on reformulating the learning objective,~\emph{e.g.}, by encouraging distortion-invariant representations through MaskedDenoising~\cite{DBLP:conf/cvpr/ChenG0MDWPZ23} or CLIP-based alignment~\cite{DBLP:journals/corr/abs-2403-15132}. A third direction leverages noise adaptation, where pretrained models are fine-tuned for specific OOD conditions~\cite{DBLP:conf/cvpr/Kim0B24, DBLP:journals/corr/abs-2412-04727, DBLP:journals/corr/abs-2503-21377}.
While effective to some extent, these methods primarily focus on data, objectives, or adaptation protocols, leaving the network architecture itself largely unexamined. In contrast, we argue that OOD robustness should be viewed as an inductive property of the network design rather than merely a byproduct of training strategies.

The core challenge in generalizing from uniform to non-uniform noise lies in the entanglement of features with spatially varying noise level maps. Ideally, a denoising model should ``factor out" the influence of the noise level. However, most existing architectures (see Fig.~\ref{fig:intro_show}) are ill-equipped to do so, lacking an explicit mechanism to disentangle this interaction. Motivated by this, we revisit a fundamental design principle, namely \emph{scale equivariance}, which is a property that ensures consistent behavior under input scaling and is closely related to first-order homogeneity.
We find that enforcing scale equivariance within the network architecture can suppress the impact of spatially varying noise levels, leading to improved generalization and train-test consistency. To this end, we design operations and modules that promote scale-equivariant behavior.

Modern denoising networks often adopt normalization layers to stabilize training. However, conventional normalization schemes (\emph{e.g.}, BatchNorm, LayerNorm) are scale-independent and can inadvertently amplify the entanglement between feature maps and noise levels. We address this by proposing a simple Constant Scaling (CS) operation, which normalizes features by the number of channels, reducing variance without introducing scale distortion.
Moreover, pixel-wise OOD noise introduces inconsistent distortions across feature dimensions. To mitigate this issue, we introduce the Normalized Self-Modulator (NSM), which transforms pixel features into a normalized space and modulates them using their original values to maintain scale equivariance. Together, CS and NSM constitute our Heterogeneous Normalization Module (HNM).
Another critical component is the activation function. Standard nonlinearities, especially those involving exponential terms, can disrupt scale properties. Apart from common activations, the star operation/gating mechanism ~\cite{DBLP:conf/cvpr/0005DBW024} achieves high-dimensional nonlinear mapping in a low-dimensional space but is typically second-order homogeneous. To retain first-order scale equivariance, we modify and extend the gating design via introducing the Interactive Gating Module (IGM). This module uses a dual-signal scaling strategy, allowing information flow to adapt jointly to both feature and gating signals, thereby enhancing selective filtering under non-uniform noise conditions.

Our major contributions are summarized as follows:
\begin{itemize}
     \item We identify \emph{scale equivariance} as a key role for mitigating the train-test inconsistency caused by non-uniform noise levels. Guided by this insight, we design a robust blind denoising network tailored for OOD scenarios.

    \item  We propose the Heterogeneous Normalization Module, which integrates feature scaling and normalized self-modulation to strike a stable feature representation and flexible feature correction. Moreover, the Interactive Gating Module is proposed to provide powerful nonlinear representations and improve the feature selection.

    \item Extensive experiments reveal the superior denoising robustness of our network against OOD noises, especially in high-level noise conditions. Our model consistently outperforms existing approaches with advanced perceptual quality across diverse degradation patterns.
    
\end{itemize}

\section{Related Work}
\label{sec:RW}
This part will briefly review representative works related to deep learning-based and robust image denoising.

\noindent  \textbf{Deep learning-based image denoising.}
Deep learning has become the dominant paradigm in image denoising, with early successes driven by Convolutional Neural Networks (CNNs)\cite{DBLP:conf/eccv/ZamirAKHKYS20, DBLP:conf/eccv/ChangLFX20, DBLP:journals/tip/ZhangLLSKF21, DBLP:conf/aaai/ShenZZ23, ghasemabadi2024cascadedgaze}, and more recent advances powered by Vision Transformer (ViT) architectures~\cite{DBLP:conf/cvpr/WangCBZLL22, DBLP:conf/cvpr/ZamirA0HK022, DBLP:conf/cvpr/LiFXDRTG23, DBLP:conf/iclr/ZhangZGDKY24, DBLP:conf/eccv/ChenLPLZQD24}, which leverage sophisticated self-attention mechanisms and normalization strategies. While these designs excel on ID data, their supervised nature makes them prone to overfitting, thus limiting their OOD robustness. 
To address this, alternative paradigms such as self-supervised learning~\cite{DBLP:conf/cvpr/ZhangZJF23, DBLP:conf/eccv/LiaoZZZR24, DBLP:conf/aaai/LiZZ25, DBLP:conf/cvpr/LiuFXM25} and zero-shot denoising~\cite{DBLP:conf/cvpr/QuanCPJ20, DBLP:conf/cvpr/MansourH23, DBLP:journals/pami/MaJZLLM25} have emerged. These approaches enable model training directly on unpaired or noisy data without requiring clean targets, for offering greater flexibility. However, they often fall short in denoising quality and still face OOD limitations, as their training remains implicitly tied to the noise distribution in the available data.
Thus, a fundamental challenge remains, saying how to retain the high performance of supervised frameworks while overcoming their inherent lack of robustness under OOD noise. This question continues to motivate much of the current research in robust image denoising.

\noindent  \textbf{Robust image denoising.}
Robust image denoising refers to a model's ability to maintain high performance when exposed to unseen or OOD noise patterns. Existing efforts to enhance such robustness can be broadly categorized as follows.
One common approach is to enrich the training data by simulating complex noise, such as via mixture noise modeling~\cite{DBLP:conf/cvpr/GuoY0Z019, DBLP:journals/ijautcomp/ZhangLLCZTFTG23} or adversarial perturbations~\cite{DBLP:conf/cvpr/ChenG0MDWPZ23}. Although these strategies help improve diversity during training, they also risk contaminating the evaluation setup by inadvertently exposing the model to OOD-like noise during training, thus complicating generalization assessment.
A second direction involves learning degradation-invariant representations, encouraging models to focus on clean image content rather than the noise distribution. Representative methods include MaskedDenoising and CLIPDenoising~\cite{DBLP:journals/corr/abs-2403-15132}. However, these methods often come at the cost of reduced image quality, particularly under high noise levels.
A third class of methods emphasizes noise adaptation. For instance, LAN~\cite{DBLP:conf/cvpr/Kim0B24} and noise translation techniques~\cite{DBLP:journals/corr/abs-2412-04727} attempt to map OOD noise into a known or ID domain. However, such methods may suffer from scalability limitations or require prior exposure to real-world noise distributions. MID~\cite{DBLP:journals/corr/abs-2503-21377} presents another strategy by fine-tuning a denoiser trained solely on simple synthetic noise using self-supervised objectives. While effective, this approach relies on ensemble inference, making it computationally intensive.

In contrast to the mentioned schemes that primarily focus on data augmentation, objective design, or adaptation techniques, this work takes a fundamentally different perspective, \emph{i.e.}, how \emph{network architecture itself} can serve as an inductive bias for robustness.

\section{Problem Analysis}
The canonical model for denoising~\cite{DBLP:conf/iccvw/LiangCSZGT21, DBLP:conf/cvpr/LiFXDRTG23}~\footnote{Without loss of generality, we adopt it through this work.} models an observed noisy image $Y$ as the sum of its underlying clean signal $X$ and an additive white Gaussian noise component $\lambda\cdot N$:
\begin{align}
Y=X+\lambda\cdot N, N_{ij}\sim \mathcal{N} (0,\mathrm {1}),
\end{align}
where $\lambda$ is a scalar representing a uniform noise level. 
While this model is widely adopted for training deep denoisers, it introduces a critical mismatch with real-world noise characteristics, which are rarely spatially uniform. In practice, due to sensor response variations and environmental influences (\emph{e.g.}, lighting, exposure, ISO), the noise level varies across spatial locations. Consequently, the scalar $\lambda$ should be replaced by a spatially varying map $\Lambda$, leading to a distributional shift between training and inference. This train-test discrepancy from uniform to non-uniform noise lies at the heart of poor OOD generalization. We identify a core reason for this degradation as feature entanglement. Conventional denoising networks tend to learn features that are tightly coupled to the magnitude of the noise. When the noise level becomes spatially varying, this coupling breaks down, leading to destabilized representations and degraded performance.
To achieve robustness under such conditions, the key is to decouple feature representations from the noise level map. We argue that this can be systematically addressed by enforcing scale equivariance in the network design. Also known as first-order homogeneity, scale equivariance ensures that the network’s output scales linearly with the input, formally expressed as follows:
\begin{definition}[Scale Equivariance]
\label{def:homogeneity}
A function $\mathcal{H}$ obeying the $\mathcal{H}(0)=0$ is said to be first-order homogeneous if, for input $I$ and a scalar $k > 0$, it satisfies the following condition:
\[
\mathcal{H}(k\cdot I) = k\cdot \mathcal{H}(I).
\]
\end{definition}

\noindent For a denoising network $\mathcal{F}$ endowed with the above property, its behavior during training under a uniform noise level $\lambda$ can be described as follows: 
\begin{align}
\mathcal{F} (X+\lambda\cdot N)=\lambda\cdot \mathcal{F} (\frac{X}{\lambda }+N)\quad \stackrel{\text{target}}{\longrightarrow}\quad {X}.
\label{train}
\end{align}
As shown in Eq.~\eqref{train}, the network no longer learns a direct mapping from a noisy image to its clean counterpart. Instead, the effective learning task becomes denoising the input $X/\lambda+N$ to produce $X/\lambda$, thereby removing only the additive noise $N$. This formulation decouples the denoising mechanism from the global noise magnitude, embedding a scale-equivariant inductive bias directly into the network architecture. This decoupling extends naturally to the more challenging OOD setting with spatially variant noise $\Lambda$:
\begin{align}
\mathcal{F}(X+\Lambda \odot  \tilde{N})=\Lambda \odot \mathcal{F}(\frac{X}{\Lambda}+\tilde{N})\stackrel{\text{target}}{\longrightarrow} {X},
\label{test}
\end{align}
where $\odot$ denotes Hadamard product and $\tilde{N}$ is a base noise term that may originate from a different distribution. 
Remarkably, the network’s internal processing remains structurally consistent: regardless of whether the noise level is uniform or spatially varying, the effective input is normalized to unit noise variance across the image. This normalization aligns with recent observations~\cite{DBLP:conf/cvpr/ChenG0MDWPZ23, DBLP:journals/corr/abs-2403-15132} that denoisers inherently prioritize learning the noise pattern over image content.
By enforcing a consistent computational pattern across domains, scale equivariance provides a principled and architecture-level path to achieving robustness.

\begin{figure*}[t]
\centering
\includegraphics[scale=0.61]{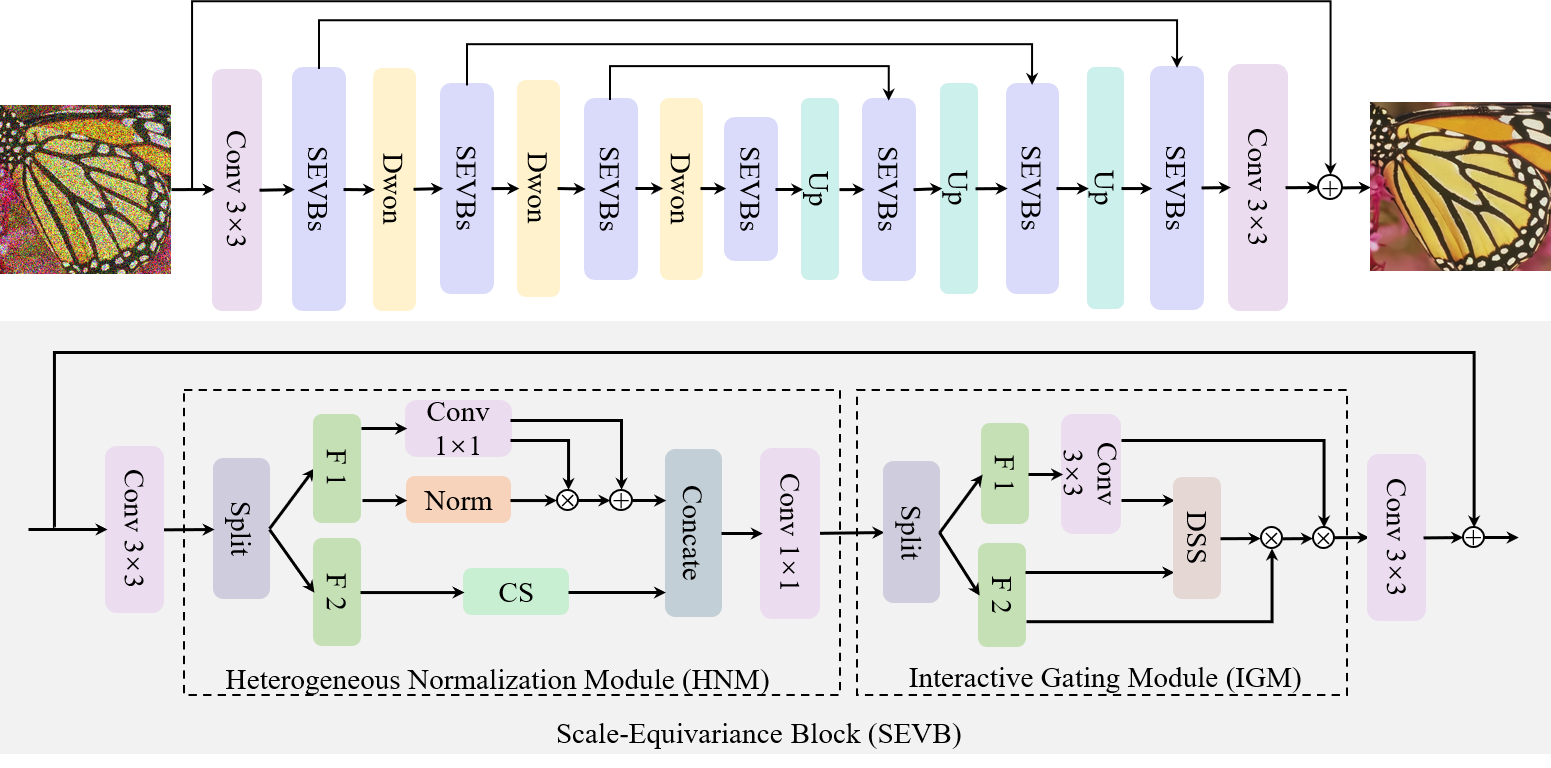}
\caption{Overview of the proposed SEVNet with scale-equivariant design. CS represents Constant Scaling. DSS involves the dual-signal scaling term. The split operation divides the feature map into equally shaped sub-maps along the channel dimension.}

\vspace{-15pt}
\label{fig:framework}
\end{figure*}

The importance of strict first-order homogeneity is best understood by examining common violations, which we term architectural pathologies for this task: \textit{1) Zero-Order Homogeneity (Scale Independence).} Consider a network with this property,~\emph{i.e.}, $\mathcal{F} (X+\lambda\cdot N)=\mathcal{F} (X/\lambda+N)$, its learning objective becomes mapping $X/\lambda + N$ to $X$. This imposes a flawed dual task: the network must simultaneously denoise the input and learn the global scaling $\lambda$, which forces an entanglement between features and the noise level seen during training. This entanglement is precisely what fails when $\lambda$ becomes a spatially-varying map; \textit{2) Higher-Order Homogeneity.} A network with second-order homogeneity, ~\emph{i.e.}, $\mathcal{F} (X+\lambda \cdot N)=\lambda^2 \cdot \mathcal{F} (X/\lambda+N)$ involves similar entanglement phenomena. The objective becomes mapping $X/\lambda + N$ to $X/\lambda^2$, creating a distorted target. Moreover, the second-order property involves a quadratic growth in feature variance,  increasing risk of learning collapse;  \textit{3) Exponential Activations.} Many exponential-based activation functions like GELU break homogeneity. They introduce input-dependent scaling errors that compound throughout the network, rendering OOD performance fragile.

\noindent \textbf{A First-Principles Validation.} With scale equivariance equipped, even a simple network should generalize well under distributional shifts.
To test this in blind denoising, we build a baseline U-Net only from demonstrably scale-equivariant components: standard convolutions and ReLU activation. Despite its simplicity and lack of advanced modules, this network serves as a clean testbed to observe the effect of scale equivariance without confounding factors.

\begin{lemma}[Homogeneity of Basic Operators]
\label{lem:core_ops}
Both convolution and ReLU are first-order homogeneous.
\end{lemma}
\begin{proof} 
Given an input $I$ and a scalar $k > 0$, the convolution output at position $(i,j)$ is:
\begin{equation}
\nonumber
\begin{aligned}
\text{Conv}(k \cdot I)_{i,j} = k \cdot \sum_{m,n} I_{i-m, j-n} \cdot W_{m,n} = k \cdot \text{Conv}(I)_{i,j}.
\end{aligned}
\end{equation}
Similarly, the relationship $\text{ReLU}(k\cdot I)_{i,j}=k \cdot \text{ReLU}(I)_{i,j}$ also holds. Thus, the claim is established.
\end{proof}
\noindent Building on this, the entire residual block, including the skip connection, preserves this property as well.

\begin{lemma}[Homogeneity of Residual Block]
\label{lem:rb_homogeneity}
The `Conv-ReLU-Conv' residual block is first-order homogeneous.
\end{lemma}
\begin{proof}
Given Lemma~\ref{lem:core_ops}, we have:
\begin{equation} \label{eq:rb_proof}
\nonumber
\begin{aligned}
\text{RB}(k\cdot I)&=\text{Conv}(\text{ReLU}(\text{Conv}(k\cdot I))) + k\cdot I \\
&  = k\cdot(\text{Conv}(\text{ReLU}(\text{Conv}(I))) + I)=k\cdot \text{RB}(I).
\end{aligned}
\end{equation}
The property holds for the entire block.
\end{proof}
\noindent As reported in Tab.~\ref{tab:preliminary}, this simple baseline outperforms those sophisticated but non-equivariant models on OOD tasks while remaining competitive on ID cases. The comparison corroborates our theoretical analysis. 

\begin{table}[!t]
\centering
\footnotesize
\setlength{\tabcolsep}{0.32mm}{

\begin{tabular}{lcccc} 
\toprule
& \multicolumn{1}{c}{\textbf{In-Distribution}} & \multicolumn{3}{c}{\textbf{Out-Of-Distribution}} \\ 
\cmidrule(lr){2-2} \cmidrule(lr){3-5} 
\textbf{Method} & \textbf{Gaussian} & \textbf{Speckle} & \textbf{Poisson} &\textbf{Mixture}\\
                & $\sigma=20$ & $\sigma=90$ & $\alpha=6$ &\\
\midrule

Uformer &\underline{32.89}/0.9450 &23.47/0.7510 &22.92/0.7253 &20.07/0.5812\\

NAFNet &32.28/0.9402 &22.69/0.7047 &22.30/0.6901 &20.15/0.5995\\
GRL &32.77/0.9419 &\underline{24.15}/\underline{0.8021} &\underline{23.65}/\underline{0.7868} &\underline{21.19}/\underline{0.6990}\\
CGNet &32.85/\underline{0.9454} &22.77/0.6999 &22.32/0.6829 &20.02/0.5850\\
Xformer	&\textbf{32.99}/\textbf{0.9456}	&21.80/0.6386	&21.15/0.6096	&17.89/0.4682\\
MHNet	&32.87/0.9452	&21.70/0.6474	&21.03/0.6150 &17.56/0.4552\\
AKDT	&32.54/0.9417 &22.58/0.6774 &22.08/0.6568 &19.52/0.5468\\
CRWKV &32.39/0.9377 &20.17/0.5845 &19.60/0.5594 &16.57/0.4299\\
Baseline &32.85/0.9447 &\textbf{24.49}/\textbf{0.8286} &\textbf{24.04}/\textbf{0.8183} &\textbf{21.67}/\textbf{0.7554}\\
\bottomrule
\end{tabular}}
\caption{Quantitative comparisons in PSNR/SSIM on the Urban100 dataset. The mixture noise is composed of Speckle and Poisson. The best results are highlighted in \textbf{bold}, while the second-best ones are \underline{underlined}.}
\label{tab:preliminary}
\vspace{-18pt}
\end{table}

\section{Designing a Scale-Equivariant Network}

Our first-principles validation demonstrated that even a simple equivariant network can outperform sophisticated yet non-compliant ones. To further advance performance, we now introduce several novel modules designed to solve key challenges in constructing a high-performance, scale-equivariant denoiser. These components collectively define our SEVNet architecture, as illustrated in Fig. \ref{fig:framework}. Due to the page limit, formal proofs of equivariance for all modules are provided in the appendix.

\noindent\textbf{Constant Scaling.} Modern networks often rely on normalization operations, \emph{e.g.}, Layer Norm, to stabilize training. While effective, these operations are scale-independent and thus conflict with our equivariant objective. To resolve this, we propose Constant Scaling (CS) as an equivariant stabilizer applied to the feature map $F\in \mathbb{R}^{c \times h \times w} $:

\begin{align}
\text{CS}(F)=\frac{F}{\sqrt[]{c} }\odot  {\eta}  ,
\end{align}
where $\eta \in \mathbb{R}^{c} $ is a learnable, per-channel parameter initialized to 1. This operation normalizes the initial feature magnitude (preventing uncontrolled variance growth), while allowing $\eta$ to adapt channel-wise amplitudes during training.

\begin{table*}[t]
\centering
\footnotesize
\setlength{\tabcolsep}{0.56mm}{
\begin{tabular}{lccccccc}
\toprule
& \multicolumn{1}{c}{\textbf{In-Distribution}} & \multicolumn{6}{c}{\textbf{Out-Of-Distribution}} \\ 
\cmidrule(lr){2-2} \cmidrule(lr){3-8} 
\textbf{Method} &\textbf{Gaussian} &\textbf{Speckle} & \textbf{Poisson}   &\textbf{Mixture}  &\textbf{Variant 1} &\textbf{Variant 2} &\textbf{Variant 3}\\

\midrule
\multicolumn{8}{c}{Kodak24} \\
\midrule

SwinIR~\cite{DBLP:conf/iccvw/LiangCSZGT21} & 33.50/0.9029 & 27.55/0.7931 & 27.23/0.7820 & 22.79/0.5887 & 25.99/0.7457 & 26.31/0.7541 & 25.63/0.7313 \\
Uformer~\cite{DBLP:conf/cvpr/WangCBZLL22} & 33.62/0.9062 & 27.20/0.7656 & 26.86/0.7506 & 22.27/0.5102 & 25.42/0.6950 & 25.82/0.7102 & 25.07/0.6750 \\
Restormer~\cite{DBLP:conf/cvpr/ZamirA0HK022}& \textbf{33.70}/\textbf{0.9076} & 26.12/0.6978 & 25.74/0.6807 & 20.11/0.4279 & 23.19/0.6078 & 24.01/0.6280 & 22.87/0.5809 \\
NAFNet~\cite{DBLP:conf/eccv/ChenCZS22}         & 33.42/0.9048 & 25.39/0.6681 & 25.15/0.6580 & 22.16/0.5422 & 23.49/0.6087 & 23.99/0.6227 & 23.33/0.5967 \\
GRL~\cite{DBLP:conf/cvpr/LiFXDRTG23}& 33.52/0.9025 & \underline{27.58}/\underline{0.7981} & \underline{27.29}/\underline{0.7904} & \underline{23.66}/\underline{0.6789} & \underline{26.19}/\underline{0.7643} & \underline{26.47}/\underline{0.7684} & \underline{25.87}/\underline{0.7560} \\
CGNet~\cite{ghasemabadi2024cascadedgaze} & 33.65/\underline{0.9075} & 25.43/0.6564 & 25.11/0.6416 & 21.74/0.4981 & 23.32/0.5953 & 23.88/0.6106 & 23.04/0.5714 \\
Xformer~\cite{DBLP:conf/iclr/ZhangZGDKY24}	&33.66/0.9069	&25.29/0.6388	 &24.82/0.6165	&18.91/0.3507	&22.56/0.5433	&23.15/0.5582	&22.10/0.5143 \\
MHNet~\cite{DBLP:journals/pr/GaoZYD25}	&33.64/0.9074	&24.88/0.6390	&24.43/0.6163	&18.67/0.3462	&21.91/0.5361	&22.72/0.5573	&21.64/0.5121\\
AKDT~\cite{DBLP:conf/visigrapp/BrateanuBAO25}	&33.48/0.9050	&25.57/0.6565	&25.20/0.6391	&20.94/0.4493	&23.39/0.5887	&23.94/0.6018	&23.05/0.5637\\
ESC~\cite{DBLP:journals/corr/abs-2503-06671}	&33.51/0.9028	&26.52/0.7017	&26.16/0.6843	&21.58/0.4691	 &24.44/0.6261	&24.94/0.6402	&24.03/0.6021 \\
CRWKV~\cite{DBLP:journals/corr/abs-2505-02705}	&33.50/0.9038	&23.89/0.5978	&23.43/0.5763	&17.68/0.3380	&20.73/0.5026 	&21.48/0.5186 	&20.26/0.4726 \\
Ours           & \underline{33.68}/0.9069 & \textbf{27.94}/\textbf{0.8164} & \textbf{27.66}/\textbf{0.8106} & \textbf{24.23}/\textbf{0.7368} & \textbf{26.64}/\textbf{0.7908} & \textbf{26.91}/\textbf{0.7932} & \textbf{26.34}/\textbf{0.7858} \\

\midrule
\multicolumn{8}{c}{Urban100} \\
\midrule

SwinIR~\cite{DBLP:conf/iccvw/LiangCSZGT21}	 &32.71/0.9415	& \underline{25.53}/0.8338	& 25.19/0.8238	 &20.78/0.6490	 &23.80/0.7858	&\underline{26.09}/0.8529 	 &23.75/0.7831\\
Uformer~\cite{DBLP:conf/cvpr/WangCBZLL22}	 &32.89/0.9450	& 25.03/0.8055	& 24.71/0.7936	 &20.07/0.5812	&23.11/0.7408	 &25.60/0.8331 	&23.10/0.7374\\
Restormer~\cite{DBLP:conf/cvpr/ZamirA0HK022}	&\textbf{33.12}/\textbf{0.9469}	& 24.31/0.7547	& 24.00/0.7439	 &18.56/0.5142	&21.65/0.6759	&24.27/0.7797	&21.64/0.6722\\
NAFNet~\cite{DBLP:conf/eccv/ChenCZS22}	&32.28/0.9402 	&23.87/0.7430	&23.64/0.7356	&20.15/0.5995	&22.05/0.6832	&22.44/0.7000	&22.04/0.6824\\
GRL~\cite{DBLP:conf/cvpr/LiFXDRTG23}	&32.77/0.9419 	& \underline{25.53}/\underline{0.8362}	& \underline{25.21}/\underline{0.8279}	&\underline{21.19}/\underline{0.6990}	&\underline{23.91}/\underline{0.7973}	&26.08/\underline{0.8541}	&\underline{23.85}/\underline{0.7952} \\
CGNet~\cite{ghasemabadi2024cascadedgaze}	&32.85/0.9454	&24.13/0.7449	&23.85/0.7361	&20.02/0.5850	&22.14/0.6835	 &22.56/0.7012 	 &22.09/0.6790\\
Xformer~\cite{DBLP:conf/iclr/ZhangZGDKY24}	&\underline{32.99}/\underline{0.9456}	&23.74/0.7138	&23.37/0.6996	&17.89/0.4682	&21.24/0.6326	 &21.69/0.6467	&21.19/0.6276\\
MHNet~\cite{DBLP:journals/pr/GaoZYD25}	&32.87/0.9452	&23.51/0.7224	&23.14/0.7073	&17.56/0.4552	&20.81/0.6290	&21.41/0.6505	&20.84/0.6262\\
AKDT~\cite{DBLP:conf/visigrapp/BrateanuBAO25}	&32.54/0.9417	&24.05/0.7304	&23.76/0.7197	&19.52/0.5468	&22.06/0.6697	&22.46/0.6841	&22.01/0.6657\\
ESC~\cite{DBLP:journals/corr/abs-2503-06671}	 &32.97/0.9439	&24.38/0.7463	&24.07/0.7351	&19.42/0.5445	&22.25/0.6793	&22.64/0.6917	&22.18/0.6747\\
CRWKV~\cite{DBLP:journals/corr/abs-2505-02705}	&32.39/0.9377	&22.00/0.6544	&21.67/0.6410	&16.57/0.4299	&19.42/0.5743	&19.94/0.5889	&19.34/0.5675\\
Ours	&32.94/0.9453 	&\textbf{26.14}/\textbf{0.8644}& \textbf{25.84}/\textbf{0.8590}	&\textbf{21.94}/\textbf{0.7698}	&\textbf{24.59}/\textbf{0.8371}	&\textbf{26.66}/\textbf{0.8771}	&\textbf{24.54}/\textbf{0.8357}\\
\bottomrule
\end{tabular}}
\caption{Quantitative denoising comparisons (PSNR/SSIM) on the different datasets. The average results across 3 noise levels for Speckle and Poisson noise are reported. The best results are marked in \textbf{bold}, while the second-best ones are \underline{underlined}.}
\label{tab:average_psnr_ssim_total}
\vspace{-19pt}
\end{table*}

\noindent\textbf{Normalized Self-Modulator.}  
While CS ensures global stability, it does not address local feature distortions caused by noise with spatial amplitude variations. To counteract this, we propose the Normalized Self-Modulator (NSM), a dynamic pixel-wise module designed to provide localized feature correction while preserving scale equivariance. The core idea is to perform a temporary local normalization, followed by an adaptive rescaling that restores critical magnitude in an equivariant manner. NSM performs self-modulation using affine parameters derived from the token's original pre-normalized state. This can be formulated as:

\begin{equation} 
\begin{aligned}
\left [ \gamma, \beta  \right ]&\leftarrow\text{Conv}_{1\times 1}(F), \\
\text{NSM}(F)&= \gamma \odot \frac{X-\mu(F) }{\sqrt{\sigma^2(F)}}+\beta, 
\end{aligned}    
\end{equation} 
 where $\mu$ and $\sigma^2$ are the per-token mean and variance, respectively. Since the affine parameters $\gamma \in\mathbb{R} ^{c\times h \times w}$ and $\beta \in\mathbb{R} ^{c\times h \times w}$ are derived via a linear transformation of the input, they are first-order homogeneous. This ensures that the entire NSM remains strictly scale-equivariant, providing adaptive feature correction without violating our principle.

\noindent\textbf{Heterogeneous Normalization Module.} While individual CS and NSM modules provide stabilization, they offer a trade-off between efficiency and expressive power. To get the best of both worlds, we combine them into a single Heterogeneous Normalization Module (HNM). The design is simple yet effective: we partition the feature channels and process one subset with the lightweight CS for baseline stability, while the other subset is processed by the more powerful NSM for dynamic, fine-grained correction. This heterogeneous strategy creates a robust, multi-level stabilization mechanism that achieves adaptive feature modulation.

\noindent\textbf{Interactive Gating Module.} Although ReLU satisfies the requirement of scale equivariance, its limited expressiveness may hinder performance. In contrast, more sophisticated activations such as GELU and Swish offer stronger nonlinear modeling capabilities but inherently violate scale equivariance, introducing a potential bottleneck in equivariant design. To bridge this gap, we propose the Interactive Gating Module (IGM), a novel activation function that is both highly expressive and strictly scale-equivariant by construction. Specifically, given the input feature map $F \in\mathbb{R} ^{c\times h \times w}$, we first split it along the channel dimension into two equal parts $F_1 \in\mathbb{R} ^{\frac{c}{2}\times h \times w}$ and $F_2 \in\mathbb{R} ^{\frac{c}{2}\times h \times w}$, and then feed them into a gating unit. We define the IGM as follows:
\begin{equation} 
\begin{aligned}
&[F_1,F_2]\leftarrow\text{Split}(F), \ 
F_{v}\leftarrow\text{Conv}_{3\times 3}(F_1),\ F_{g}\leftarrow F_2,\\
&\text{IG}(F_v,F_g)=\frac{F_v\odot F_g}{\sqrt[]{\sigma^2 (F_v)+\sigma^2 (F_g)} },
\label{IGM}
\end{aligned}
\end{equation}
where  $\sigma^2$ stands for the per-token variance. A brilliance of our IGM is the dual-signal scaling term, which not only ensures the entire module remains scale-equivariant but also enables the gating strength to be co-determined by both the feature $F_v$ and the gating signal $F_g$. This is achieved by transforming the gating signal's variance into a relative importance weight $\frac{\sigma^2(F_g)}{\sigma^2(F_v)+\sigma^2(F_g)}$. Beyond its mathematical compliance, the IGM is interactive. The denominator normalizes the output based on the combined energy (variance) of both the value and the gate. This prevents the output from exploding when both signals are strong, allowing the gate's influence to be modulated by its own magnitude relative to the value signal. This stable interaction enhances the network's feature selection capability beyond a simple ReLU, providing the desired equivariant non-linearity.

\section{Experiments}
Our SEVNet is built on the principle of scale equivariance. Together, these components create a network that is theoretically sound for OOD generalization. We now turn to the experimental results to verify these claims.

\begin{figure*}[!t]
    \centering
        \vspace{-7pt}
        \subfigure[Poisson Noise]{\includegraphics[width=\textwidth]{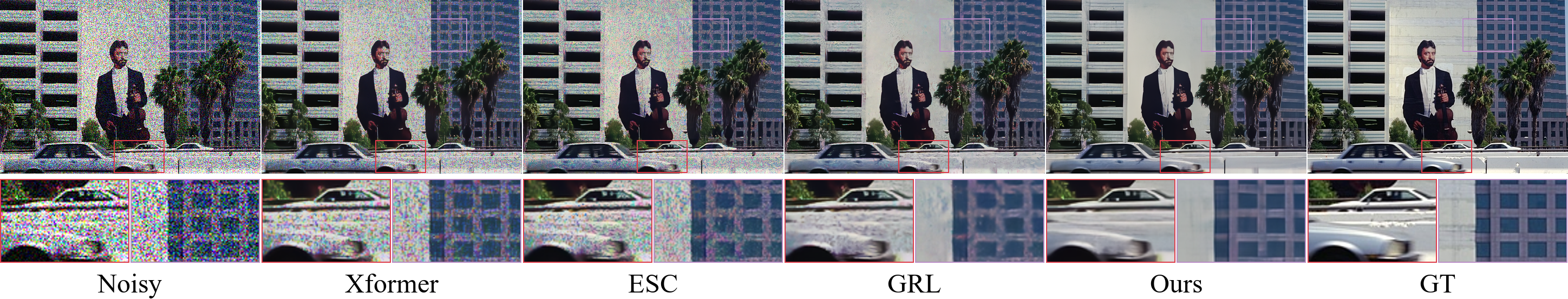}}
        \vspace{-7pt}
        \subfigure[Variant 2 of Speckle Noise]{\includegraphics[width=\textwidth]{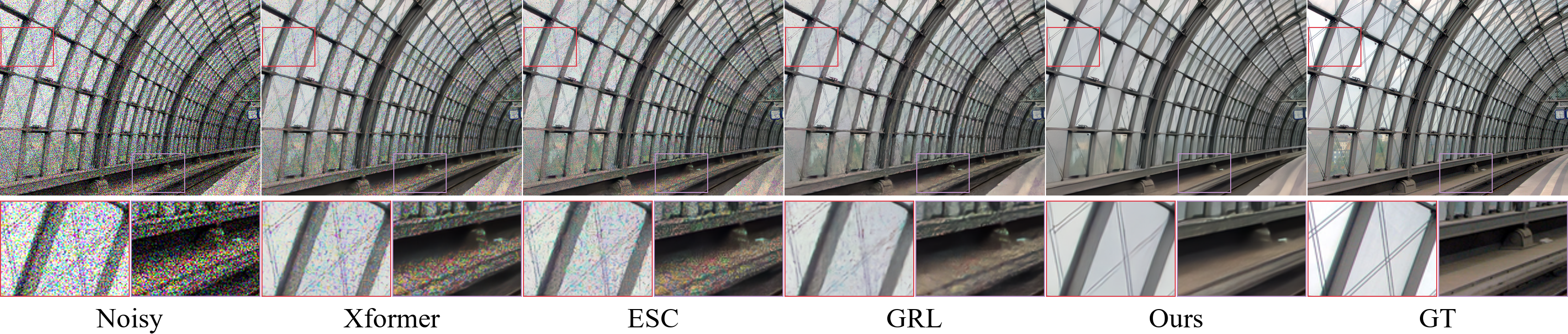}}
        \vspace{-7pt}
        \subfigure[Variant 3 of Speckle Noise]{\includegraphics[width=\textwidth]{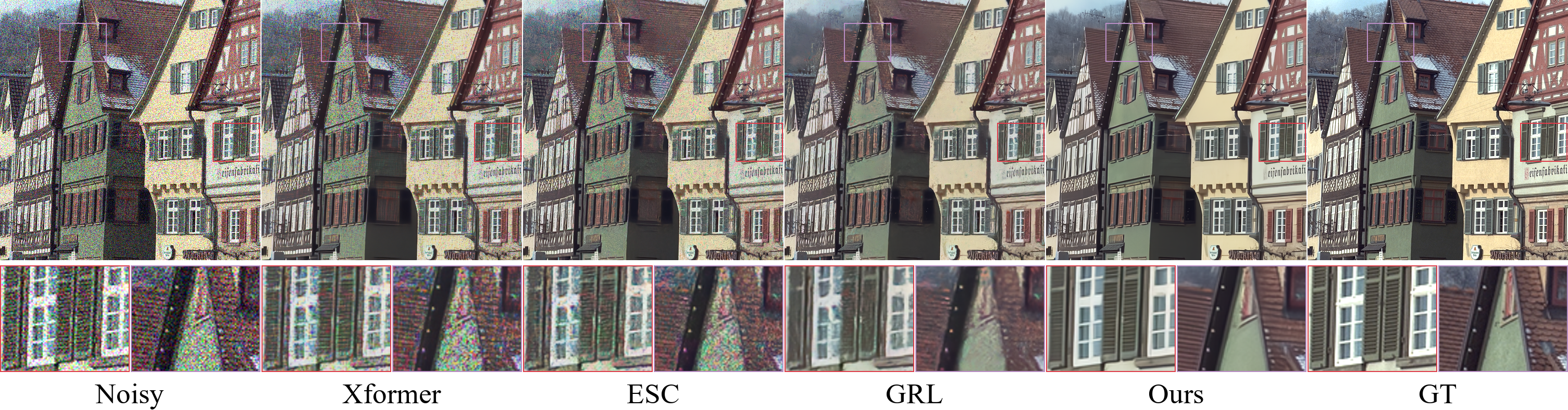}}
    \caption{Qualitative comparisons on OOD noises. Please refer to the appendix to see more visual results.}
    \vspace{-17pt}
    \label{fig:visual case} 
\end{figure*}

\begin{table*}[ht]
\centering
\footnotesize 
\setlength{\tabcolsep}{1.8mm}{
\begin{tabular}{lccccccccc}
\toprule 

    & \multicolumn{1}{c}{\textbf{Speckle}}
    & \multicolumn{1}{c}{\textbf{Poisson}} 
    & \multicolumn{1}{c}{\textbf{Mixture}}     
    & \multicolumn{1}{c}{\textbf{Variant 1}}
    & \multicolumn{1}{c}{\textbf{Variant 2}} 
    & \multicolumn{1}{c}{\textbf{Variant 3}}      \\ 

    & $\sigma=90$ 
    & $\alpha=60$
    & $(\sigma=90, \alpha=6)$
    & 
    & 
    & 
    \\ 
    \textbf{Setting}     & PSNR/ SSIM      & PSNR/SSIM  & PSNR/SSIM  & PSNR/SSIM    & PSNR/SSIM  & PSNR/SSIM\\ 
\midrule
\multicolumn{8}{c}{Heterogeneous Normalization Module (HNM)} \\
\midrule
(1a) Ours &\textbf{26.86}/\textbf{0.7933} &\textbf{26.42}/\textbf{0.7838}&\textbf{24.23}/\textbf{0.7368} &\textbf{26.64}/\textbf{0.7908} 	&\textbf{26.91}/\textbf{0.7932} &\textbf{26.34/}\textbf{0.7858}  \\
(1b) HNM$\to $ Full CS  &26.78/0.7888  &26.35/0.7793 &24.17/0.7317 &26.57/0.7858 &26.83/0.7887 &26.27/0.7812\\
(1c) w/ CS + w/o NSM   &26.74/0.7885 &26.30/0.7791 &24.16/0.7322 &26.52/0.7855 &26.79/0.7885 &26.23/0.7810\\

(1d) w/ HNM, IGM$\to $ELU &26.77/0.7894 &26.33/0.7799 &24.16/0.7325 &26.55/0.7866 &26.82/0.7893 &26.25/0.7818\\
(1e) w/ HNM, IGM$\to $GELU &26.74/0.7872 &26.29/0.7775 &24.13/0.7291 &26.52/0.7844 &26.79/0.7871 &26.23/0.7794\\
(1f) w/o HNM, IGM$\to$ELU &22.79/0.5859 &21.51/0.5205 &16.63/0.2782 &21.56/0.5565 &22.42/0.5896 &20.84/0.5186\\
(1g) w/o HNM, IGM$\to$GELU &25.73/0.7241 &25.17/0.7007 &22.41/0.5495 &25.40/0.7119 &25.74/0.7228 &24.97/0.6938\\
(1h) Scale-independent affine &26.54/0.7778  &26.09/0.7666 &23.85/0.7101 &26.32/0.7747  &26.59/0.7778 &26.01/0.7684 \\

\midrule
\multicolumn{8}{c}{Interactive Gating Module (IGM)} \\
\midrule
(2a) w/o IGM  &26.75/0.7887  &26.31/0.7789 &24.15/0.7317 &26.53/0.7858 &26.80/0.7886 &26.24/0.7811\\
(2b) IGM $\to $ ReLU &26.80/0.7892 &26.36/0.7797 &24.17/0.7317 &26.59/0.7865 &26.85/0.7891 &26.28/0.7814\\
(2c) Dual$\to $Single scaling &26.83/0.7916  &26.39/0.7821 &24.20/0.7347 &26.62/0.7891 &26.88/0.7916 &26.31/0.7840\\
(2d) HNM$\to$ LN, w/ IGM&26.39/0.7684 &25.90/0.7545 &23.46/0.6744  &26.13/0.7636  &26.41/0.7676 &25.83/0.7562\\
(2e) HNM$\to $LN, IGM$\to$ReLU&25.53/0.6979 &24.82/0.6612 &21.88/0.5057 &25.15/0.6898 &25.54/0.7034  &24.68/0.6640\\
(2f) HNM$\to$LN, IGM$\to$GELU &26.24/0.7583  &25.78/0.7426 &23.07/0.5991 &25.96/0.7485 &26.28/0.7576 &25.65/0.7388\\
\bottomrule
\end{tabular}
}
    \caption{Ablation studies on the Kodak24 dataset.}
    \label{tab:ablation}
\vspace{-17pt}
\end{table*}

\noindent \textbf{Datasets}
We train models on the DIV2K~\cite{DBLP:conf/cvpr/TimofteAG0ZLSKN17} dataset with 
Gaussian noise of various levels ($\sigma \in \left [5,20 \right ]$) for synthetic OOD noise removal, and ISP noise ~\cite{DBLP:conf/cvpr/GuoY0Z019} for handling real-world cases. We evaluate models on signal-dependent noise (Poisson, Speckle, and their mixtures). To enhance the non-uniformity of the level map, variants (Variant 1-3) of Speckle noise are synthesized by spatial functions~\cite{DBLP:journals/pami/YueYZZMW24}: sine-cosine, peaks, and Gaussian kernels. The evaluation is conducted on standard test sets (Kodak24, CBSD68, Urban100) and real-world benchmarks (CC, PolyU, HighISO)~\cite{DBLP:journals/corr/abs-2412-04727}. Denoising quality is quantified by the Peak Signal-to-Noise Ratio (PSNR) and the Structural Similarity Index (SSIM). 

\begin{table}[H]
\centering
\footnotesize 
\setlength{\tabcolsep}{1.4mm}{
\begin{tabular}{lccccccccc}
\toprule 
\textbf{Method} &  \textbf{CC} & \textbf{PolyU} & \textbf{HighISO} &\\ 
       &  PSNR/SSIM &  PSNR/SSIM  &  PSNR/SSIM\\
\midrule

SwinIR	&33.61/0.8583 &36.02/0.9151  &35.53/0.8797\\
Uformer	&34.30/0.8766  &36.18/0.9182  &35.96/0.8891\\
Restormer	&33.71/0.8576 &36.02/0.9146 &35.57/0.8788\\
NAFNet&34.88/0.8952 &36.30/0.9230  &36.38/0.9002\\
GRL	&33.92/0.8638  &36.14/0.9181   &35.75/0.8842 \\	
CGNet	&33.69/0.8580 &36.05/0.9155 &35.58/0.8802\\
Xformer     &33.65/0.8567   &36.01/0.9141  &35.54/0.8786 \\		
MHNet       &33.83/0.8634     &36.05/0.9156  &35.70/0.8826\\		
AKDT &34.48/0.8838 &36.17/0.9197  &36.05/0.8921\\	
ESC &34.08/0.8697  &36.30/0.9208 &35.94/0.8880\\
CRWKV     &31.13/0.7507   &31.98/0.7718  &31.66/0.7571\\
Ours+NT	&\textbf{36.95}/\textbf{0.9486} &\textbf{37.60}/\textbf{0.9557} &\textbf{39.36}/\textbf{0.9675}\\
\bottomrule
\end{tabular}
}
    \caption{Quantitative comparisons on real-world noise.}
    \label{tab:real-world}
\vspace{-17pt}
\end{table}

\noindent \textbf{Quantitative Comparison.}
As shown in Tab.~\ref{tab:average_psnr_ssim_total}, our scale-equivariant network exhibits superior OOD robustness without compromising ID performance. Despite incorporating advanced modules and popular components, these existing networks violate scale equivariance, leading to severe performance degradation in OOD scenarios. More results and efficiency comparisons can be found in the appendix.

\noindent \textbf{Qualitative Comparison.}
As shown in Fig.~\ref{fig:visual case}, even in severely degraded conditions, our network effectively removes various OOD noises while preserving clear edges and some textures. In contrast, the denoising results of competing methods are poor and exhibit different degradation patterns,~\emph{e.g.}, residual noise and artifacts. 

\noindent \textbf{Real-World Extension.}
The applicability of our scale-equivariant network, which assumes independent noise, can be successfully extended to real-world data by integrating a lightweight, SEVB-based noise translator~\cite{DBLP:journals/corr/abs-2412-04727}. The strong performance of this approach (Tab.~\ref{tab:real-world}), surpassing complex yet inefficient models, highlights a highly promising path for synthetic-to-real generalization that involves shifting the focus from building ever-more-complex denoisers to combining simple, principled networks, \textit{i.e.}, those that are scale-equivariant, with noise adaptation methods. 

\noindent \textbf{Ablation Study.}
In Tab.~\ref{tab:ablation}, we validate the effectiveness of the HNM and IGM, as well as the role of scale equivariance:

\noindent \textit{1) Impact of HNM}. Applying CS to the entire feature mapping (1b) or bypassing NSM (1c) leads to a performance drop, confirming the necessity of combining both components. Interestingly, HNM allows the network to maintain comparable generalization even when using exponential activations like ELU (1d) and GELU (1e). Without HNM, however, these activations (1f-1g) severely impair OOD capability, likely because HNM shifts features to intervals where the exponential effect is less pronounced. For the NSM's design, generating its affine parameters from the normalized features (1h), also degrades performance, though the decline is moderate since the NSM operates on only a portion of the features. Collectively, these results (1e-1h) demonstrate that scale equivariance is critical for robust denoising and that our HNM can partially suppress the adverse effects of violating this principle.

\begin{figure}[t]
\centering
\includegraphics[width=\linewidth]{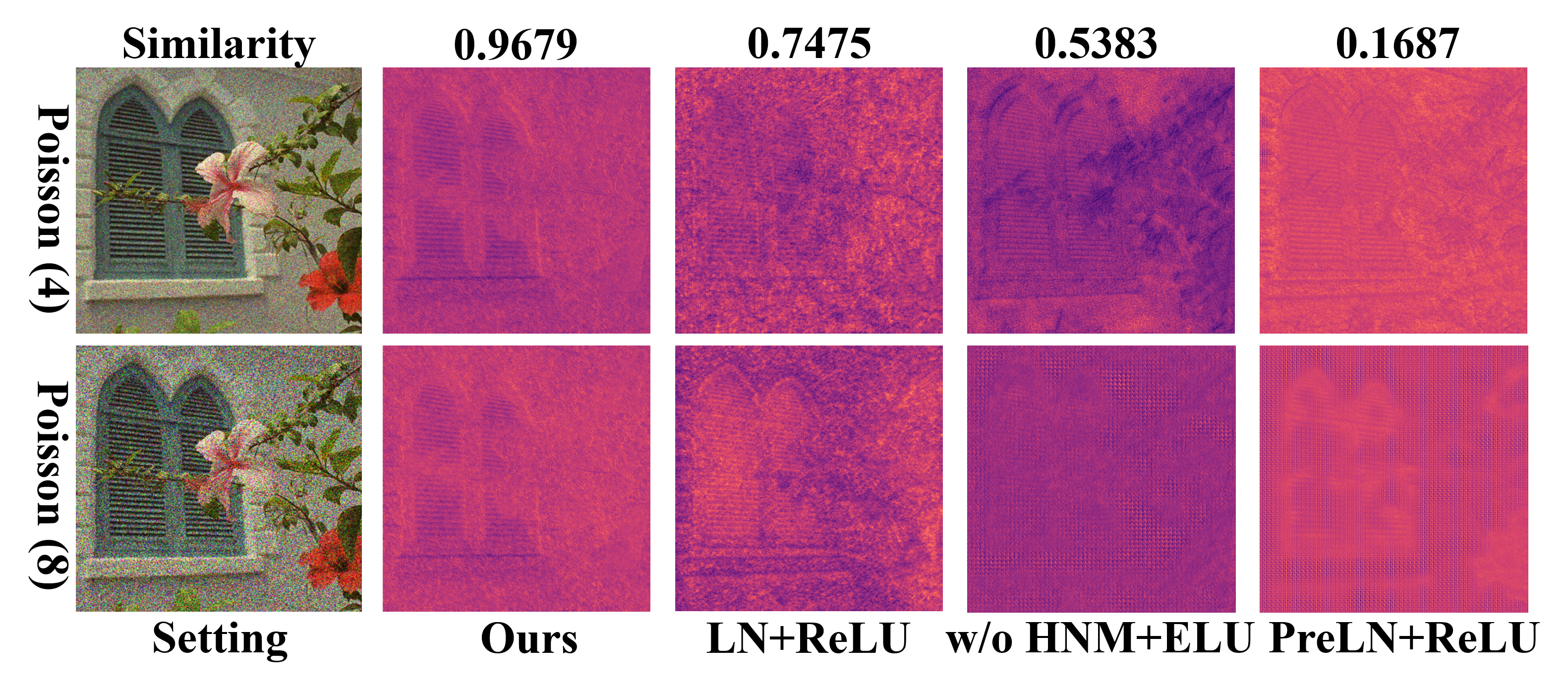}
\caption{Comparison in visual and cosine similarities of feature maps under different noise level maps.} 
\label{fig:feature analysis}
\vspace{-20pt}
\end{figure}

\noindent  \textit{2) Impact of IGM}. Removing IGM (2a) or replacing it with ReLU (2b) both yield suboptimal results, demonstrating IGM's superior nonlinear capabilities. compared with the scaling term obtained from gating (2c), the dual-signal scaling term further enhances OOD robustness through its feature selection effect. Under the scale-independent Layer Norm (LN) setting, our scaling term (2d) enables the network to perceive the scale but does not give accurate equivariance, thereby retaining weak generalization. This resilience is not found in standard activations like ReLU (2e) and GELU (2f). More importantly, without the scaling term, IGM becomes a second-order homogeneous operation, and the resulting quadratic growth in feature variance causes training to collapse, making it impossible to report the corresponding ablation results. Ultimately, these results (2d-2f) demonstrate that a scale-equivariant design is key for OOD denoising and that our IGM partially mitigates the performance degradation caused by scale-invariant operations.

\noindent \textbf{Feature Analysis.}
To further explain the advantage of scale equivariance, we visualize feature maps (Fig.~\ref{fig:feature analysis}) from the final decoder stage. Under vastly different corruption degrees, our scale-equivariant network produces remarkably consistent feature maps.  Conversely, designs that violate this principle heighten the network's sensitivity to variations in the level map. For instance, a network with scale-independence (middle and preceding Layer Norm) or exponential activation (ELU) yields dissimilar feature pairs. In short, Scale equivariance indeed forces the network to learn robust representations that are decoupled from the noise level map.

\section{Conclusion}
We have demonstrated that scale equivariance is a key architectural principle for OOD denoising, allowing networks to generalize from uniform to non-uniform noise level maps. Guided by this principle, we develop several network components. Specifically, HNM stabilizes and corrects feature distributions. IGM provides powerful nonlinear representations and refines the transmission of information flow. Our work not only delivers a superior OOD denoiser but also advocates for a new perspective in designing deep networks for robust image restoration.  Future work could explore this principle for more spatially non-uniform restoration tasks.

\section{Acknowledgement}
This work was supported by the National Natural Science Foundation of China under Grant nos.62372251 and 62072327. 
\bibliography{aaai2026}

\clearpage

\appendix
\section{Appendix}
\subsection{Theoretical Analysis}
Following the design principle of scale equivariance, we propose the Constant Scaling (CS), Normalized Self-Modulator (NSM), and Interactive Gating Module (IGM) to construct a more powerful OOD denoiser. The proofs of equivariance for the main steps within these network components are provided below. To facilitate the proofs, we uniformly represent the input as vectors, as the key operations are performed on pixel feature vectors.

\begin{lemma}[Homogeneity of CS]
\label{lem:cs_homogeneity}
The Constant Scaling operation is first-order homogeneous.
\end{lemma}
\begin{proof}
Given an input $\mathbf{v} \in \mathbb{R}^c  $ and a scalar $k > 0$,  we need to show that $\text{CS}(k \cdot \mathbf{v}) = k \cdot \text{CS}(\mathbf{v})$.

Since the learnable parameter $\eta \in \mathbb{R}^c $ is independent of the input $\mathbf{v}$, it remains unchanged when the input is scaled. We proceed as follows:
\begin{align}
\text{CS}(k \cdot \mathbf{v}) &= \frac{k \cdot \mathbf{v}}{\sqrt{c}} \odot \eta \nonumber= k \cdot \text{CS}(\mathbf{v}). \nonumber
\end{align}
This completes the proof.
\end{proof}

\begin{lemma}[Homogeneity of NSM]
\label{lem:nsm_homogeneity}
The Normalized Self-Modulator (NSM) operation is first-order homogeneous.
\end{lemma}

\begin{proof}
Given an input $\mathbf{v} \in \mathbb{R}^c $ and a scalar $k > 0$, we need to show that $\text{NSM}(k \cdot \mathbf{v}) = k \cdot \text{NSM}(\mathbf{v})$.

First, we analyze the components:
\begin{itemize}
    \item \textbf{Affine Parameters:} The $1 \times 1$ convolution is a linear transformation and thus satisfies first-order homogeneity:
    \begin{equation}
\left [ k\cdot \beta, k\cdot \gamma \right ]\leftarrow \text{Conv}_{1 \times 1}(k\cdot \mathbf{v} ). \nonumber
    \end{equation}

    \item \textbf{Mean and Standard Deviation:} The mean and variance of the scaled vector $k \cdot \mathbf{v}$ are:
    \begin{align}
        \mu(k \cdot \mathbf{v}) &= \frac{1}{c}\sum_{i=1}^{c} (k \cdot \mathbf{v}_i) = k \cdot \mu(\mathbf{v}). \nonumber\\
        \sigma^2(k \cdot \mathbf{v}) &= \frac{1}{c}\sum_{i=1}^{c} (k \cdot \mathbf{v}_i - k \cdot \mu(\mathbf{v}))^2 \nonumber \nonumber\\
        &= \frac{k^2}{c}\sum_{i=1}^{c} (\mathbf{v}_i - \mu(\mathbf{v}))^2 = k^2 \cdot \sigma^2(\mathbf{v}).\nonumber
    \end{align}
    Therefore, the standard deviation is $\sigma(k \cdot \mathbf{v}) = k \cdot \sigma(\mathbf{v})$.\nonumber

    \item \textbf{Normalized Vector:} The normalized vector part of NSM is scale-independent (zero-order homogeneous):
    \begin{equation}
        \frac{k \cdot \mathbf{v} - \mu(k \cdot \mathbf{v})}{\sigma(k \cdot \mathbf{v})} = \frac{\mathbf{v} - \mu(\mathbf{v})}{\sigma(\mathbf{v})}.\nonumber
    \end{equation}
\end{itemize}
Now, we combine these results to analyze the full NSM operation:
\begin{align}
\text{NSM}(k \cdot \mathbf{v}) &= k \cdot \gamma \odot \frac{k \cdot \mathbf{v}-\mu(k \cdot \mathbf{v}) }{\sigma(k \cdot \mathbf{v})} + k \cdot \beta \nonumber \\
&= k \cdot \text{NSM}(\mathbf{v}).\nonumber
\end{align}
This completes the proof.
\end{proof}

\begin{lemma}[Homogeneity of IGM]
\label{lem:igm_homogeneity}
The Interactive Gating Module (IGM) is a first-order homogeneous system.
\end{lemma}

\begin{proof}
Given an input $\mathbf{v}\in \mathbb{R}^c $ and a scalar $k > 0$, we need to show that $\text{IGM}(k \cdot \mathbf{v}) = k \cdot \text{IGM}(\mathbf{v})$.


We will analyze the homogeneity of the numerator and the denominator of the IGM formula separately. Here, the convolution and split operations are omitted

\begin{itemize}
    \item \textbf{Numerator Analysis:} The numerator is a Hadamard product. Its homogeneity is determined as follows:
    \begin{equation}
        (k \cdot \mathbf{v}_v) \odot (k \cdot \mathbf{v}_g) = k^2 \cdot (\mathbf{v}_v \odot \mathbf{v}_g). \nonumber
    \end{equation}
    Thus, the numerator is second-order homogeneous.

    \item \textbf{Denominator Analysis:} The denominator is the square root of the sum of variances. The variance of a scaled vector $\sigma^2(k \cdot \mathbf{v})$ is second-order homogeneous. Applying this property to the denominator, we get:
    \begin{align}
        \sigma^2(k \cdot \mathbf{v}_v) + \sigma^2(k \cdot \mathbf{v}_g)  
        &= k^2 \cdot \left (  \sigma^2(\mathbf{v}_v) + \sigma^2(\mathbf{v}_g)\right ) . \nonumber
    \end{align}
\end{itemize}

Combining the results for the numerator and the denominator, we analyze the full IGM operation:
\begin{align}
\text{IGM}(k \cdot \mathbf{v}_v, k \cdot \mathbf{v}_g) &= \frac{(k \cdot \mathbf{v}_v) \odot (k \cdot \mathbf{v}_g)}{\sqrt{\sigma^2(k \cdot \mathbf{v}_v) + \sigma^2(k \cdot \mathbf{v}_g)}} \nonumber \\
&= k \cdot \text{IGM}(\mathbf{v}_v, \mathbf{v}_g). \nonumber
\end{align}
This completes the proof.
\end{proof}


\subsection{Experiments}

\begin{table*}[!t]
\centering
\footnotesize
\setlength{\tabcolsep}{2mm}{
\begin{tabular}{lcccccc}
\toprule
& \multicolumn{3}{c}{\textbf{Speckle Noise}} & \multicolumn{3}{c}{\textbf{Poisson Noise}} \\ 
\cmidrule(lr){2-4} \cmidrule(lr){5-7} 
\textbf{Method}&$\sigma=60$ &$\sigma=80$ &$\sigma=90$ & $\alpha=4$   &$\alpha=5$  &$\alpha=6$\\
\midrule
\multicolumn{7}{c}{Kodak24} \\
\midrule
SADNet~\cite{DBLP:conf/eccv/ChangLFX20}	&28.58/0.8236	&26.76/0.7703	&25.97/0.7409	&28.20/0.8138 	&26.77/0.7711	&25.53/0.7227\\
SwinIR~\cite{DBLP:conf/iccvw/LiangCSZGT21}	&\underline{29.11}/\underline{0.8368}	&27.22/0.7866	&26.33/0.7560	&\underline{28.71}/\underline{0.8272} 	&27.21/0.7859	&25.78/0.7328 \\
Uformer~\cite{DBLP:conf/cvpr/WangCBZLL22}	&28.97/0.8326	&26.83/0.7567	&25.81/0.7075	&28.53/0.8202	&26.82/0.7563	&25.22/0.6752 \\
Restormer~\cite{DBLP:conf/cvpr/ZamirA0HK022}	&28.21/0.7837 	&25.71/0.6839	&24.45/0.6257	&27.74/0.7670	 &25.75/0.6847	&23.74/0.5904\\
NAFNet~\cite{DBLP:conf/eccv/ChenCZS22}	&26.75/0.7182	&25.06/0.6563	&24.36/0.6297 	 &26.38/0.7040	&25.07/0.6557	&23.99/0.6144 \\
GRL~\cite{DBLP:conf/cvpr/LiFXDRTG23}	&29.08/0.8360	&\underline{27.23}/\underline{0.7900}	&\underline{26.42}/\underline{0.7682}	&28.68/0.8265	&\underline{27.22}/\underline{0.7897}	&\underline{25.96}/\underline{0.7549}\\
CGNet~\cite{ghasemabadi2024cascadedgaze}	&27.08/0.7213	&25.02/0.6411 	&24.19/0.6068 	&26.61/0.7017 	 &25.01/0.6383 	&23.72/0.5848 \\
Xformer~\cite{DBLP:conf/iclr/ZhangZGDKY24}	&27.84/0.7554	&24.74/0.6145	&23.28/0.5466	&27.25/0.7298	&24.74/0.6132 	&22.46/0.5065 \\
MHNet~\cite{DBLP:journals/pr/GaoZYD25}	&27.11/0.7399	&24.43/0.6212	&23.10/0.5559	 &26.56/0.7157	&24.42/0.6183	&22.32/0.5149\\
AKDT~\cite{DBLP:conf/visigrapp/BrateanuBAO25}	&27.49/0.7389	&25.12/0.6382 	&24.09/0.5923 	&26.99/0.7175	&25.11/0.6362 	&23.51/0.5637\\
ESC~\cite{DBLP:journals/corr/abs-2503-06671}	&28.46/0.7851	&26.08/0.6850	&25.01/0.6351 	&27.98/0.7656	&26.10/0.6834	&24.41/0.6040\\
CRWKV~\cite{DBLP:journals/corr/abs-2505-02705}	&26.50/0.7053	&23.29/0.5742 	&21.87/0.5139	&25.86/0.6785	&23.33/0.5721	&21.11/0.4782\\
Ours	 &\textbf{29.35}/\textbf{0.8463}	&\textbf{27.61}/\textbf{0.8096} 	&\textbf{26.86}/\textbf{0.7933}	&\textbf{28.97}/\textbf{0.8385}	&\textbf{27.60}/\textbf{0.8094}	&\textbf{26.42}/\textbf{0.7838}\\

\midrule
\multicolumn{7}{c}{Urban100} \\
\midrule
SADNet~\cite{DBLP:conf/eccv/ChangLFX20}	&26.06/0.8468	&23.95/0.7910	&23.08/0.7617	&25.68/0.8380	 &24.03/0.7932	&22.67/0.7458\\

SwinIR~\cite{DBLP:conf/iccvw/LiangCSZGT21}	&\underline{27.37}/\underline{0.8794} 	&\underline{25.10}/0.8264	 &24.11/0.7955	&\underline{26.90}/\underline{0.8699}	&\underline{25.11}/0.8266 	 &23.56/0.7750 \\
Uformer~\cite{DBLP:conf/cvpr/WangCBZLL22}	&27.05/0.8699	&24.56/0.7956	&23.47/0.7510	&26.58/0.8583	&24.62/0.7971	&22.92/0.7253\\
Restormer~\cite{DBLP:conf/cvpr/ZamirA0HK022}	&26.54/0.8294	&23.82/0.7419	&22.56/0.6927	&26.08/0.8172	&23.94/0.7470	&21.98/0.6675\\
NAFNet~\cite{DBLP:conf/eccv/ChenCZS22}	 &25.43/0.7915 	&23.49/0.7328	&22.69/0.7047	&25.07/0.7817	&23.55/0.7351	&22.30/0.6901 \\
GRL~\cite{DBLP:conf/cvpr/LiFXDRTG23}	&27.34/0.8787	&25.09/\underline{0.8277}	 &\underline{24.15}/\underline{0.8021}	&26.87/0.8690	&\underline{25.11}/\underline{0.8279}	&\underline{23.65}/\underline{0.7868}\\
CGNet~\cite{ghasemabadi2024cascadedgaze}	&25.92/0.8019	&23.69/0.7330	&22.77/0.6999	&25.49/0.7901	 &23.75/0.7352	&22.32/0.6829 \\
Xformer~\cite{DBLP:conf/iclr/ZhangZGDKY24}	&26.26/0.8078	&23.16/0.6950	&21.80/0.6386	&25.71/0.7906	&23.26/0.6987	&21.15/0.6096\\
MHNet~\cite{DBLP:journals/pr/GaoZYD25}	&25.81/0.8121	 &23.01/0.7076	&21.70/0.6474	&25.30/0.7961	&23.09/0.7108 	&21.03/0.6150\\
AKDT~\cite{DBLP:conf/visigrapp/BrateanuBAO25}	&25.98/0.7973	&23.59/0.7164	&22.58/0.6774	&25.53/0.7835	&23.66/0.7187	&22.08/0.6568\\
ESC~\cite{DBLP:journals/corr/abs-2503-06671}	&26.53/0.8206 	&23.86/0.7307	&22.74/0.6876	&26.05/0.8059 	&23.96/0.7338	&22.21/0.6655\\
CRWKV~\cite{DBLP:journals/corr/abs-2505-02705}	&24.41/0.7439	 &21.42/0.6347	&20.17/0.5845	&23.88/0.7253	&21.54/0.6383	&19.60/0.5594\\
Ours	&\textbf{27.83}/\textbf{0.8944}	 &\textbf{25.74}/\textbf{0.8584}	&\textbf{24.84}/\textbf{0.8405}	&\textbf{27.40}/\textbf{0.8877}	&\textbf{25.76}/\textbf{0.8589}	&\textbf{24.36}/\textbf{0.8305}\\

\midrule
\multicolumn{7}{c}{CBSD68} \\
\midrule
SADNet~\cite{DBLP:conf/eccv/ChangLFX20}	&27.02/0.8109	&25.24/0.7532	&24.50/0.7243	&26.70/0.8011	&25.31/0.7556	 &24.15/0.7093\\
SwinIR~\cite{DBLP:conf/iccvw/LiangCSZGT21}	&\underline{28.05}/\underline{0.8367}	&\underline{26.21}/\underline{0.7864}	&25.36/0.7580	&\underline{27.67}/\underline{0.8272}	&\underline{26.21}/\underline{0.7861}	&24.86/0.7376\\
Uformer~\cite{DBLP:conf/cvpr/WangCBZLL22}	&27.74/0.8274 	&25.67/0.7559	&24.72/0.7124	&27.33/0.8145	&25.69/0.7552	&24.19/0.6843\\
Restormer~\cite{DBLP:conf/cvpr/ZamirA0HK022}	&26.74/0.7652	&24.32/0.6738	&23.19/0.6258	&26.28/0.7488	 &24.38/0.6757	&22.59/0.5984\\
NAFNet~\cite{DBLP:conf/eccv/ChenCZS22}	 &25.04/0.6973	&23.41/0.6347	&22.77/0.6084	&24.66/0.6823	&23.40/0.6330 	&22.41/0.5917\\
GRL~\cite{DBLP:conf/cvpr/LiFXDRTG23}	&28.00/0.8344	&26.18/0.7848 	&\underline{25.40}/\underline{0.7614}	&27.62/0.8245	&26.19/0.7849	&\underline{24.97}/\underline{0.7476} \\
CGNet~\cite{ghasemabadi2024cascadedgaze}	&25.81/0.7191	&23.83/0.6454 	&23.03/0.6132	&25.39/0.7032	&23.85/0.6450	&22.61/0.5943\\
Xformer~\cite{DBLP:conf/iclr/ZhangZGDKY24}	&26.52/0.7474	&23.66/0.6260	&22.37/0.5674	 &25.96/0.7251	&23.68/0.6262	&21.66/0.5336\\
MHNet~\cite{DBLP:journals/pr/GaoZYD25}	&25.72/0.7292	&23.20/0.6241	 &23.20/0.6241	&25.22/0.7085	 &23.22/0.6234 	&21.29/0.5317\\
AKDT~\cite{DBLP:conf/visigrapp/BrateanuBAO25}	&26.09/0.7295	&23.82/0.6379	&22.87/0.5962	&25.62/0.7108	&23.84/0.6377 	&22.36/0.5721\\
ESC~\cite{DBLP:journals/corr/abs-2503-06671}	&27.13/0.7766	&24.93/0.6908	&23.96/0.6488	&26.70/0.7598	&24.96/0.6902	&23.43/0.6229\\
CRWKV~\cite{DBLP:journals/corr/abs-2505-02705}	&24.17/0.6607	&21.14/0.5363	&19.89/0.4813	&23.49/0.6323	&21.11/0.5322 	&19.17/0.4461\\
Ours	&\textbf{28.17}/\textbf{0.8407}	&\textbf{26.44}/\textbf{0.7978}	&\textbf{25.70}/\textbf{0.7786}	&\textbf{27.81}/\textbf{0.8320}	&\textbf{26.46}/\textbf{0.7980}	&\textbf{25.31}/\textbf{0.7679}\\

\bottomrule
\end{tabular}}
\caption{Quantitative denoising comparisons (PSNR/SSIM) on the different datasets. The best results are marked in \textbf{bold}, while the second-best ones are \underline{underlined}.}
\label{tab:psnr_ssim_total}
\vspace{-20pt}
\end{table*}

\begin{table*}[!t]
\centering
\footnotesize
\setlength{\tabcolsep}{2mm}{
\begin{tabular}{lcccccc}
\toprule
& \multicolumn{3}{c}{\textbf{Speckle Noise*}} & \multicolumn{3}{c}{\textbf{Speckle Noise variant*}} \\ 
\cmidrule(lr){2-4} \cmidrule(lr){5-7} 
\textbf{Method}&$\sigma=60$ &$\sigma=80$ &$\sigma=90$ & Variant 1  &Variant 2  &Variant 3\\
\midrule
\multicolumn{7}{c}{Kodak24} \\
\midrule
SADNet~\cite{DBLP:conf/eccv/ChangLFX20}	&26.83/0.7717	&25.15/0.7064	&24.41/0.6701	&24.23/0.6663	&24.45/0.6730 	&23.92/0.6471 \\
SwinIR~\cite{DBLP:conf/iccvw/LiangCSZGT21}	&\underline{27.36}/\underline{0.7871}	&25.40/0.7147 	&24.49/0.6725 	&24.15/0.6625	&24.51/0.6752	&23.83/0.6441 \\
Uformer~\cite{DBLP:conf/cvpr/WangCBZLL22}	&26.92/0.7501	&24.74/0.6414	&23.81/0.5901 	&23.51/0.5845	&23.86/0.5965	 &23.24/0.5654 \\
Restormer~\cite{DBLP:conf/cvpr/ZamirA0HK022}	&25.77/0.6786 	&23.21/0.5588 	&22.12/0.5065	&21.40/0.5024 	&21.97/0.5160	&21.08/0.4791\\
NAFNet~\cite{DBLP:conf/eccv/ChenCZS22}	&25.30/0.6619	&23.92/0.6122	&23.35/0.5919	 &22.86/0.5785 	&23.21/0.5906	 &22.66/0.5665\\
GRL~\cite{DBLP:conf/cvpr/LiFXDRTG23}	&27.29/0.7845	&\underline{25.52}/\underline{0.7358}	&\underline{24.78}/\underline{0.7138}	&\underline{24.59}/\underline{0.7091}	&\underline{24.82}/\underline{0.7139}	 &\underline{24.32}/\underline{0.7014}\\
CGNet~\cite{ghasemabadi2024cascadedgaze}	&25.13/0.6382	&23.55/0.5727	&22.94/0.5467 	&22.44/0.5422	&22.82/0.5539	&22.17/0.5219 \\
Xformer~\cite{DBLP:conf/iclr/ZhangZGDKY24}	&24.95/0.6188	&21.96/0.4765	&20.79/0.4214 	 &20.42/0.4198 	&20.80/0.4277	&20.05/0.4031\\
MHNet~\cite{DBLP:journals/pr/GaoZYD25}	&24.41/0.6095	&21.75/0.4741 	 &20.63/0.4181 	&20.00/0.4132 	&20.53/0.4248 	&19.70/0.3960\\
AKDT~\cite{DBLP:conf/visigrapp/BrateanuBAO25}	&25.23/0.6368	&23.16/0.5437	&22.34/0.5057	&21.95/0.5040	&22.33/0.5125	 &21.66/0.4875 \\
ESC~\cite{DBLP:journals/corr/abs-2503-06671}	&25.95/0.6752	&23.77/0.5722 	 &22.90/0.5282 	&22.54/0.5227	&22.91/0.5323	&22.24/0.5075\\
CRWKV~\cite{DBLP:journals/corr/abs-2505-02705}	&21.95/0.4964	&19.36/0.3847	&18.37/0.3416 	 &17.90/0.3385	&18.30/0.3429	&17.58/0.3261\\
Ours	&\textbf{27.69}/\textbf{0.8088}	&\textbf{26.02}/\textbf{0.7730}	&\textbf{25.31}/\textbf{0.7579}	&\textbf{25.13}/\textbf{0.7557}	&\textbf{25.35}/\textbf{0.7576}	 &\textbf{24.87}/\textbf{0.7509} \\

\midrule
\multicolumn{7}{c}{Urban100} \\
\midrule
SADNet~\cite{DBLP:conf/eccv/ChangLFX20}	 &24.03/0.7909 	&22.30/0.7274	&21.61/0.6954 	&21.44/0.6904	&21.64/0.6978	&21.40/0.6859\\
SwinIR~\cite{DBLP:conf/iccvw/LiangCSZGT21}	&\underline{25.32}/\underline{0.8264}	 &23.19/0.7531	&22.31/0.7150	&22.09/0.7076	&22.32/0.7174 	&22.05/0.7053 \\
Uformer~\cite{DBLP:conf/cvpr/WangCBZLL22}	&24.75/0.7940	&22.48/0.6964 	 &21.58/0.6519 	&21.35/0.6467	&21.59/0.6569 	&21.34/0.6440\\
Restormer~\cite{DBLP:conf/cvpr/ZamirA0HK022}	&23.89/0.7386	&21.35/0.6349	&20.32/0.5882 	&19.84/0.5822	&20.21/0.5942 	&19.81/0.5783\\
NAFNet~\cite{DBLP:conf/eccv/ChenCZS22}	 &23.68/0.7357 	&22.06/0.6778	&21.41/0.6519 	&21.07/0.6383 	&21.32/0.6524	&21.04/0.6371\\
GRL~\cite{DBLP:conf/cvpr/LiFXDRTG23}	&25.21/0.8224	&\underline{23.24}/\underline{0.7648} 	&\underline{22.46}/\underline{0.7385}	&\underline{22.29}/\underline{0.7342} 	&\underline{22.49}/\underline{0.7389}	&\underline{22.24}/\underline{0.7322}\\
CGNet~\cite{ghasemabadi2024cascadedgaze}	 &23.83/0.7333	&22.01/0.6660 	&21.31/0.6372 	 &20.97/0.6275	&21.24/0.6414 	&20.90/0.6232\\
Xformer~\cite{DBLP:conf/iclr/ZhangZGDKY24}	&23.40/0.7013 	&20.63/0.5836	 &19.58/0.5366 	&19.31/0.5342 	&19.58/0.5425 	&19.25/0.5305\\
MHNet~\cite{DBLP:journals/pr/GaoZYD25}	&23.01/0.7000 	&20.37/0.5776 	 &19.32/0.5270	&18.86/0.5211	&19.24/0.5327 	&18.86/0.5181 \\
AKDT~\cite{DBLP:conf/visigrapp/BrateanuBAO25}	&23.74/0.7181	&21.70/0.6373	&20.90/0.6032	&20.63/0.5996	&20.88/0.6086	&20.57/0.5955\\
ESC~\cite{DBLP:journals/corr/abs-2503-06671}	&23.81/0.7271	&21.58/0.6371	&20.73/0.5991	 &20.48/0.5948	&20.72/0.6025	&20.41/0.5905\\ 
CRWKV~\cite{DBLP:journals/corr/abs-2505-02705}	&21.45/0.6315	&19.00/0.5289	&18.08/0.4880 	&17.70/0.4843	 &18.02/0.4921 	&17.61/0.4783\\
Ours	&\textbf{25.91}/\textbf{0.8593}	&\textbf{24.00}/\textbf{0.8193}	&\textbf{23.23}/\textbf{0.8005} 	&\textbf{23.06}/\textbf{0.7975}	 &\textbf{23.27}/\textbf{0.8005} 	&\textbf{23.01}/\textbf{0.7959}\\

\midrule
\multicolumn{7}{c}{CBSD68} \\
\midrule
SADNet~\cite{DBLP:conf/eccv/ChangLFX20}	&25.31/0.7543	&23.80/0.6912 	&23.16/0.6590 	 &23.02/0.6554 	&23.19/0.6624 	&22.94/0.6510 \\
SwinIR~\cite{DBLP:conf/iccvw/LiangCSZGT21}	&\underline{26.36}/\underline{0.7855}	&24.50/0.7176	&23.66/0.6805 	&23.44/0.6732	&23.64/0.6826 	&23.37/0.6695 \\
Uformer~\cite{DBLP:conf/cvpr/WangCBZLL22}	&25.79/0.7500 	&23.77/0.6535	&22.91/0.6082 	 &22.69/0.6031 	&22.91/0.6136	&22.63/0.5998\\
Restormer~\cite{DBLP:conf/cvpr/ZamirA0HK022}	&24.43/0.6703 	&22.13/0.5682	 &21.15/0.5223	&20.63/0.5166 	&21.00/0.5294 	&20.49/0.5077\\
NAFNet~\cite{DBLP:conf/eccv/ChenCZS22}	&23.52/0.6344	&22.26/0.5809 	&21.77/0.5593	&21.43/0.5500	&21.64/0.5603	&21.31/0.5434 \\
GRL~\cite{DBLP:conf/cvpr/LiFXDRTG23}	&26.26/0.7783	&\underline{24.56}/\underline{0.7260} 	&\underline{23.85}/\underline{0.7024} 	&\underline{23.72}/\underline{0.6985}	&\underline{23.86}/\underline{0.7023}	&\underline{23.64}/\underline{0.6976}\\
CGNet~\cite{ghasemabadi2024cascadedgaze}	&23.97/0.6446	 &22.42/0.5796 	&21.82/0.5528 	&21.49/0.5485 	&21.71/0.5589	 &21.32/0.5359\\
Xformer~\cite{DBLP:conf/iclr/ZhangZGDKY24}	&23.90/0.6302	&21.23/0.5053 	&20.18/0.4546	&19.90/0.4511	 &20.16/0.4603 	&19.81/0.4493\\
MHNet~\cite{DBLP:journals/pr/GaoZYD25}	&23.21/0.6153	&20.74/0.4933	 &19.70/0.4416	&19.21/0.4357 	 &19.56/0.4467	&19.14/0.4339\\
AKDT~\cite{DBLP:conf/visigrapp/BrateanuBAO25}	&23.97/0.6375	&22.04/0.5514	&21.28/0.5163	&21.01/0.5140	&21.24/0.5217 	 &20.88/0.5091\\
ESC~\cite{DBLP:journals/corr/abs-2503-06671}	&24.87/0.6843	 &22.88/0.5950	&22.08/0.5558 	&21.84/0.5504	&22.07/0.5604	&21.67/0.5440\\
CRWKV~\cite{DBLP:journals/corr/abs-2505-02705}	&21.04/0.5211 	&18.57/0.4077	&17.63/0.3637	&17.29/0.3604	&17.56/0.3674	&17.18/0.3586\\
Ours	&\textbf{26.55}/\textbf{0.7974} 	&\textbf{24.93}/\textbf{0.7549} 	&\textbf{24.25}/\textbf{0.7365}	&\textbf{24.13}/\textbf{0.7337}	&\textbf{24.27}/\textbf{0.7358} 	&\textbf{24.05}/\textbf{0.7345}\\

\bottomrule
\end{tabular}}
\caption{Quantitative denoising comparisons (PSNR/SSIM) on the different datasets. The best results are marked in \textbf{bold}, while the second-best ones are \underline{underlined}. "*" indicates that the base noise follows the standard Laplacian distribution.}
\label{tab:psnr_ssim_total*}
\vspace{-20pt}
\end{table*}

\noindent \textbf{Implementation details}. We extract $128 \times 128$ size of patches for training. The data augmentation is conducted on input patches with random horizontal flips and rotations. The batch size is set to 4. For synthetic OOD noises, we adopt the Adam solver with $\beta_1=0.9$ and $\beta_2=0.999$ for 1000 epochs. The learning rate starts from $2 \times 10^{-4}$ and decays by a factor of 0.5 every 200 epochs. For real-world cases, we first freeze the pre-trained blind denoiser. Then, we set the epoch and initial learning rate to 250 and $1 \times 10^{-4}$, respectively, to train the lightweight noise translator based on our SEVB. The corresponding training process and loss function refer to this noise adaptive method~\cite{DBLP:journals/corr/abs-2412-04727}. All the experiments are conducted with PyTorch and NVIDIA GeForce RTX 3090 GPUs. Our network and others are trained under the blind denoising task. To facilitate image processing, data centralization and its corresponding inverse operation are applied at the start and end of all networks, respectively.

\noindent \textbf{Synthetic OOD Noise.}
For the OOD scenarios, we synthesize several pixel-dependent noises, and the corresponding noise level maps are spatially non-uniform. \textit{1) Poisson noise:} $Y=X+\alpha  \cdot N_p$, where $Y$, $X$, and $N_p$ respectively represent the observed image, the clear image, and the Poisson noise component. The synthesis of $N_p$ depends on the pixels of $X$. $\alpha \in\left \{4, 5, 6 \right \} $ represents the scaling factor of the noise level. 
\textit{2) Speckle noise:} $Y=X+\sigma \cdot \sqrt{X}\odot N$, where $N$ is derived from the standard Gaussian distribution. The noise level at each position is affected by the corresponding pixel, while $\sigma \in\left \{60,80,90 \right \} $ controls the global damage intensity. 
\textit{3) Mixture Noise:} $Y=X+\sigma \cdot \sqrt{X}\odot N + \alpha  \cdot N_p$, where $\sigma=90$ and $\alpha=6$. Its noise level map and base noise are more complex. 
\textit{4) Variants of Speckle Noise.} To further enhance the level map non-uniformity of Speckle noise, we adopt spatial functions (sine-cosine, peaks, and Gaussian kernels) to construct its variants (Variant 1-3). The form of the variant is $Y_{ij}=X_{ij}+\phi_{ij}  \cdot \sqrt{X_{ij}}  \cdot N_{ij}$, where $\phi_{ij}$ represents the local scaling factor generated by the spatial position $(i,j)$. The overall value of $\phi$ is normalized within the range of $\left [60, 120  \right ] $. 
\textit{5) Extension of base noise.} We replace the $N$ in the Speckle noise with the $N_l$ sampled from the Laplace distribution, expanding the evaluation of OOD robustness.

\begin{figure*}[!t]
    \centering

        \subfigure[ Speckle Noise of $\sigma=40$ and its variant 1]{\includegraphics[scale=0.43]{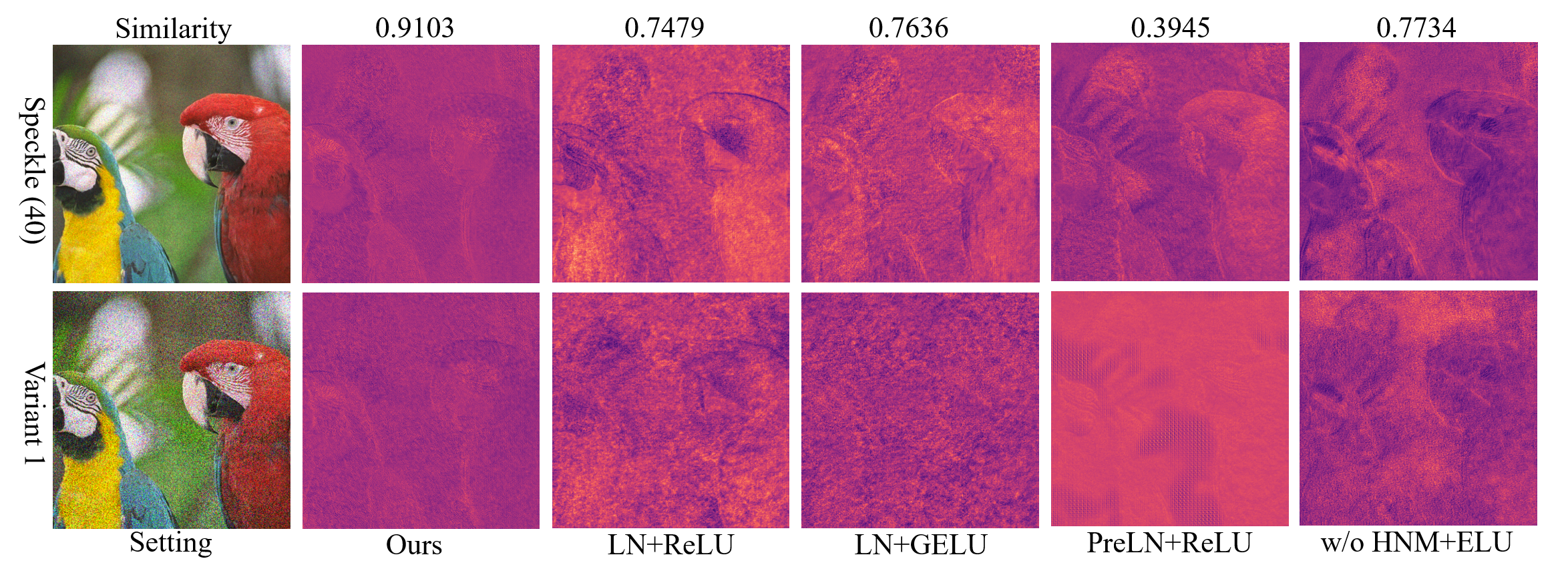}}
        \vspace{-7pt}
        \subfigure[ Speckle Noise of $\sigma=40$ and its variant 2]{\includegraphics[scale=0.43]{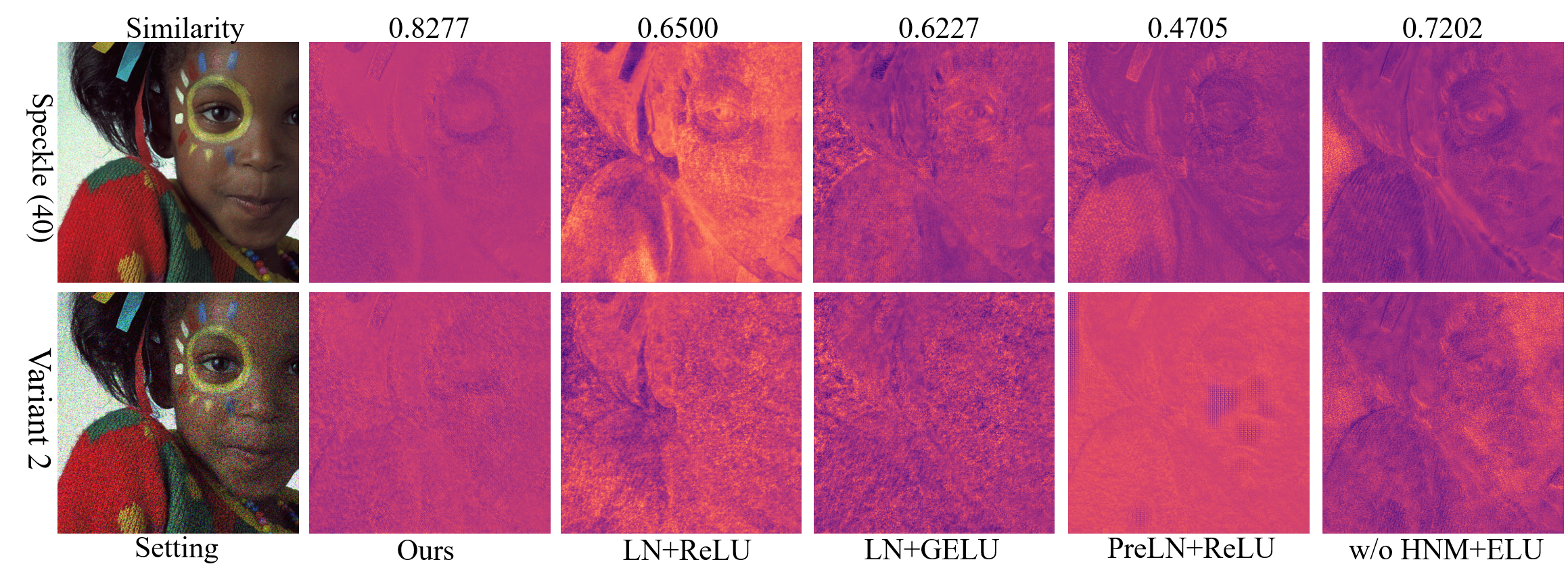}}
        \vspace{-7pt}
    \caption{Comparison in visual and cosine similarities of feature maps under different noise level maps.}
    \vspace{-17pt}
    \label{fig:feature analysis add} 
\end{figure*}

\begin{table}[!ht]
\centering
\footnotesize 
\setlength{\tabcolsep}{0.95mm}{
\begin{tabular}{lccccccccc}
\toprule 
Method & Params (M) & MACs (G) &Times (ms) &PSNR/SSIM \\ 
\midrule
SADNet	&4.23&18.97	&0.8	&25.70/0.7238 \\
SwinIR	&3.14&204.99	&53.1	&25.92/0.7325\\
Uformer	&20.63 & 41.55 	&5.1	&25.44/0.6844\\
Restormer	&26.13	&141.24	&4.8	&23.67/0.6038\\
NAFNet	&29.16& 16.25	&3.4	&23.92/0.6161\\
GRL	&5.11&316.57	&10.4	&26.18/0.7593\\
CGNet	&185.38 &115.08	&5.0	&23.75/0.5956\\
Xformer	&25.22& 143.12	&33.0	 &24.36/0.5898\\
MHNet	&67.54&127.40	&6.7	&23.98/0.5878\\
AKDT	&11.43 &56.89 	&0.7	 &25.08/0.6292\\
ESC    &0.94   &51.59   &0.5    &24.61/0.6206\\
CRWKV	&21.08& 113.55	&31.2	&23.00/0.5585\\
Ours	&31.32 &158.69	&1.4	&26.62/0.7889\\
Ours(lw)	&16.28& 76.50 	&0.7	&26.56/0.7857\\
\bottomrule
\end{tabular}
}
    \caption{Comparison of different networks in terms of complexity, per-image inference time, and average OOD evaluation on the Kodak24 dataset.}
    \label{tab:complexity}
\vspace{-20pt}
\end{table}

\noindent \textbf{Feature Analysis.}
We provide additional examples of feature comparison. As shown in Fig.~\ref{fig:feature analysis add}, our scale-equivariant network consistently captures highly similar features, even when faced with significant differences in global corruption intensity and inconsistent spatial variation patterns in the level map. Conversely, networks that violate this design principle are overly sensitive to such variations. For instance, networks incorporating scale-independent Layer Norm or exponential activations (w/o HNM+ELU) produce markedly dissimilar feature maps. Furthermore, while these non-equivariant designs may yield usable features at low noise levels, they severely distort the network's representations at high levels. particularly, placing Layer Norm at the front of a block (PreLN+ReLU) results in more severe feature distortion under high noise compared to putting it in the middle (LN+ReLU). In short, scale equivariance enables the network to generate aligned feature maps that are decoupled from the noise level map, thereby laying the foundation for robust OOD denoising.

\noindent \textbf{Quantitative Comparison.}
We provide more detailed quantitative comparisons for Speckle and Poisson noise. As shown in Tab.~\ref{tab:psnr_ssim_total}, our scale-equivariant network consistently achieves superior OOD denoising performance across various corruption levels. Conversely, existing methods, which violate the principle of scale equivariance, suffer from a lack of robustness, even when equipped with advanced modules. Furthermore, when the random distribution used to synthesize Speckle noise is replaced from Gaussian to Laplacian, our scale-equivariant network still demonstrates OOD denoising advantage, as shown in Tab.~\ref{tab:psnr_ssim_total*}.  This indicates its robustness to changes in the base noise distribution.

\noindent \textbf{Qualitative Comparison.}
As shown in  Fig.~\ref{fig:mixture} to Fig.~\ref{fig:variant3*}, our method produces clean denoised images, restoring clear edges and partial texture information, even at high levels of corruption. In contrast, affected by the non-uniform noise level map, the denoising capabilities of existing methods are spatially inconsistent and exhibit different patterns of degradation,~\emph{e.g.}, coarse results, residual noise, and artifacts.

\noindent \textbf{Real-World Extension.}
The principle of scale equivariance assumes independently distributed noise, a condition not met by spatially correlated real-world camera noise. To bridge this gap, we construct the SEVB-based noise translator trained on synthetic ISP noise. This translator maps correlated noise to a distribution that our pretrained denoiser can process.
As shown in Fig.~\ref{fig:cc1} and Fig.~\ref{fig:highiso1},  denoising networks trained directly on ISP noise fail to generalize, whereas our approach demonstrates outstanding robustness.   This suggests that our scale-equivariant network offers a powerful denoising prior, and combining it with noise adaptation provides an effective path for generalizing from synthetic to real data.

\noindent \textbf{Complexity Comparison}
According to Tab.~\ref{tab:complexity}, our model offers fast inference, though its complexity is not the lowest. We observe that OOD robustness does not correlate with model complexity, as both more and less complex networks perform poorly. To confirm this, we build a lightweight version of our network that retains comparable robustness with only half the complexity. This result confirms that OOD robustness is primarily determined by the scale-equivariant design, not model capacity.

\begin{figure*}[t]
\centering
\includegraphics[scale=0.16]{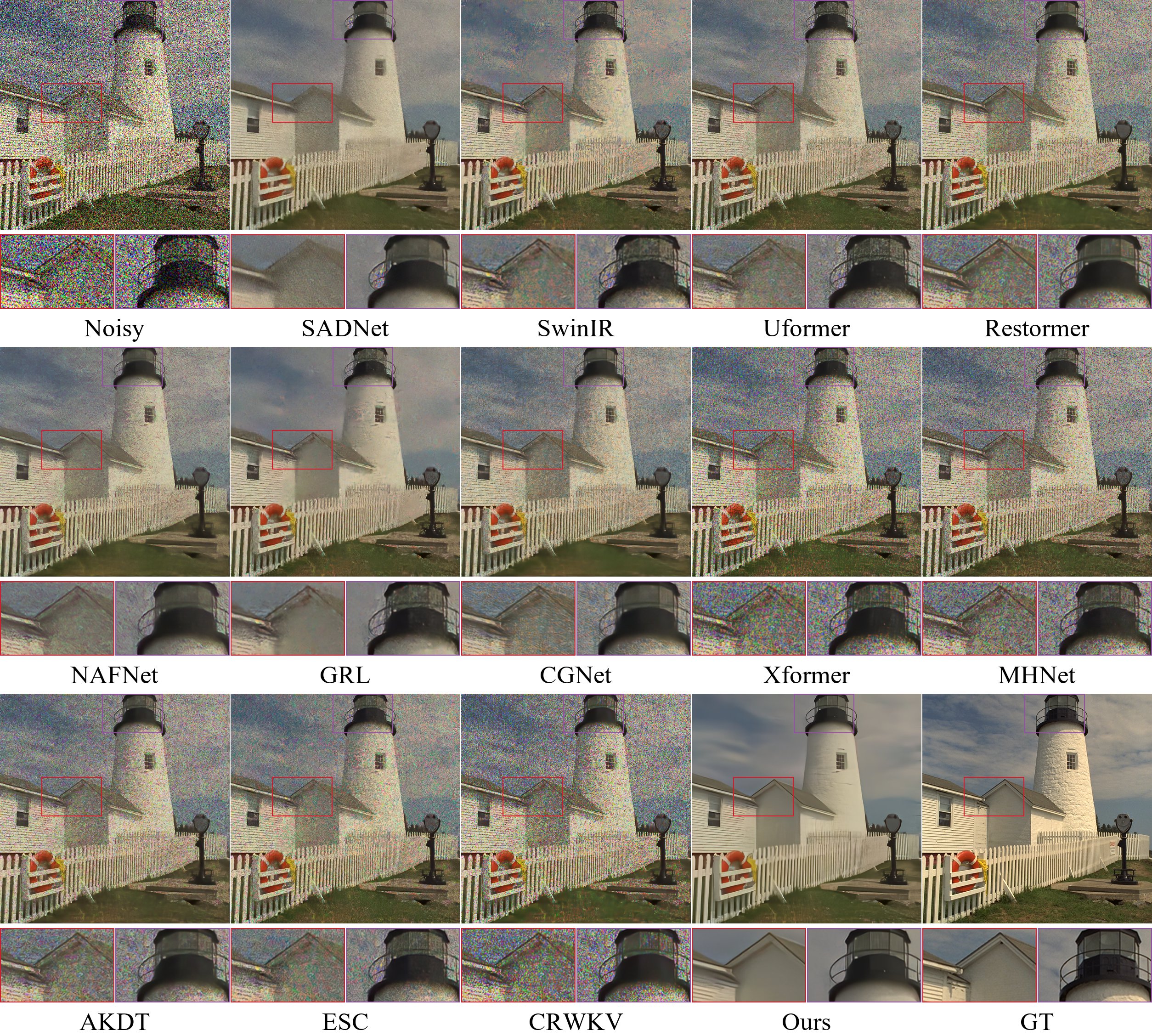}
\caption{Qualitative comparisons on mixture noise of $\sigma=90$ and $\alpha=6$.}
\vspace{-10pt}
\label{fig:mixture}
\end{figure*}

\begin{figure*}[t]
\centering
\includegraphics[scale=0.16]{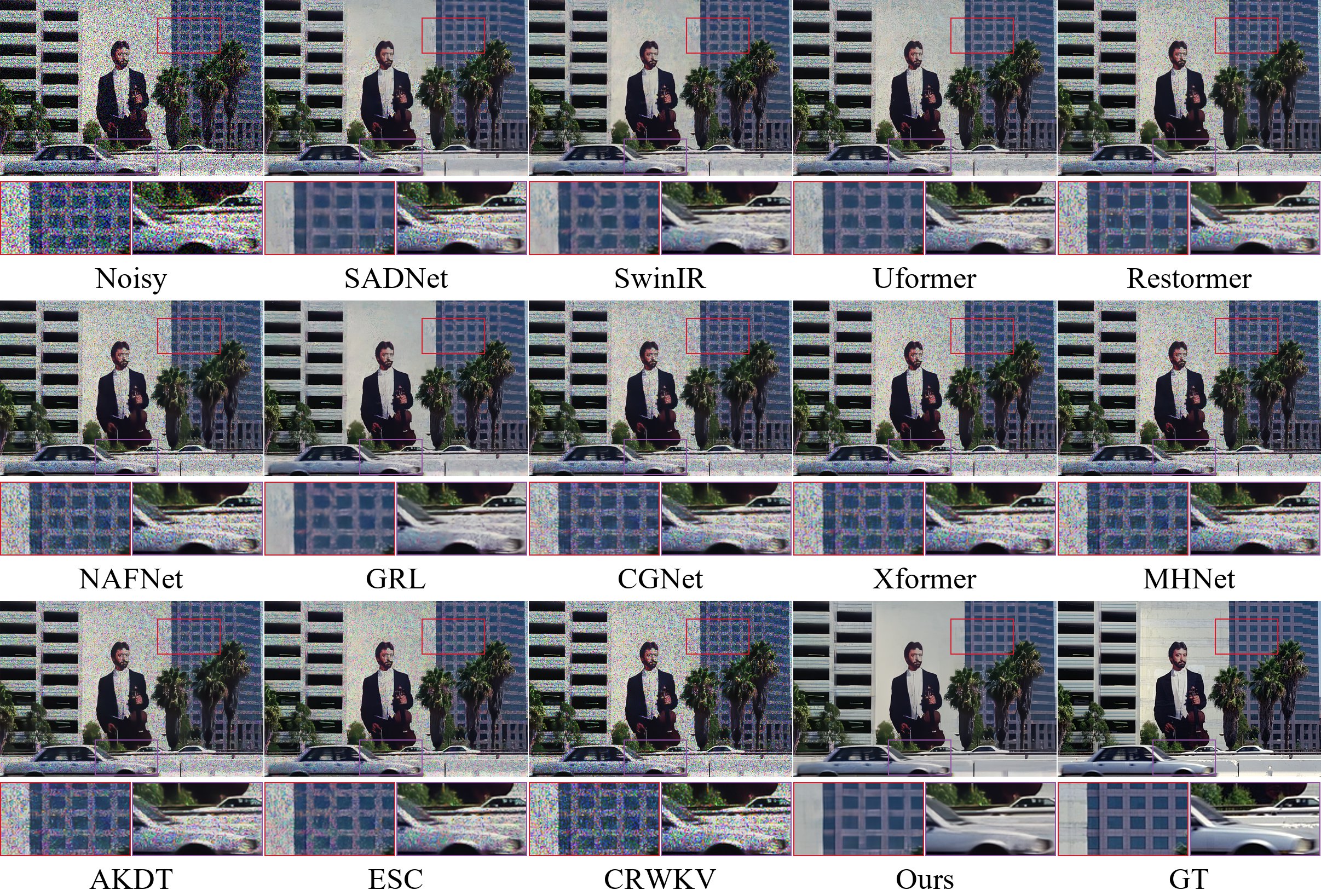}
\caption{Qualitative comparisons on Poisson noise of $\alpha=6$.}
\vspace{-20pt}
\label{fig:poisson(6)}
\end{figure*}

\begin{figure*}[t]
\centering
\includegraphics[scale=0.14]{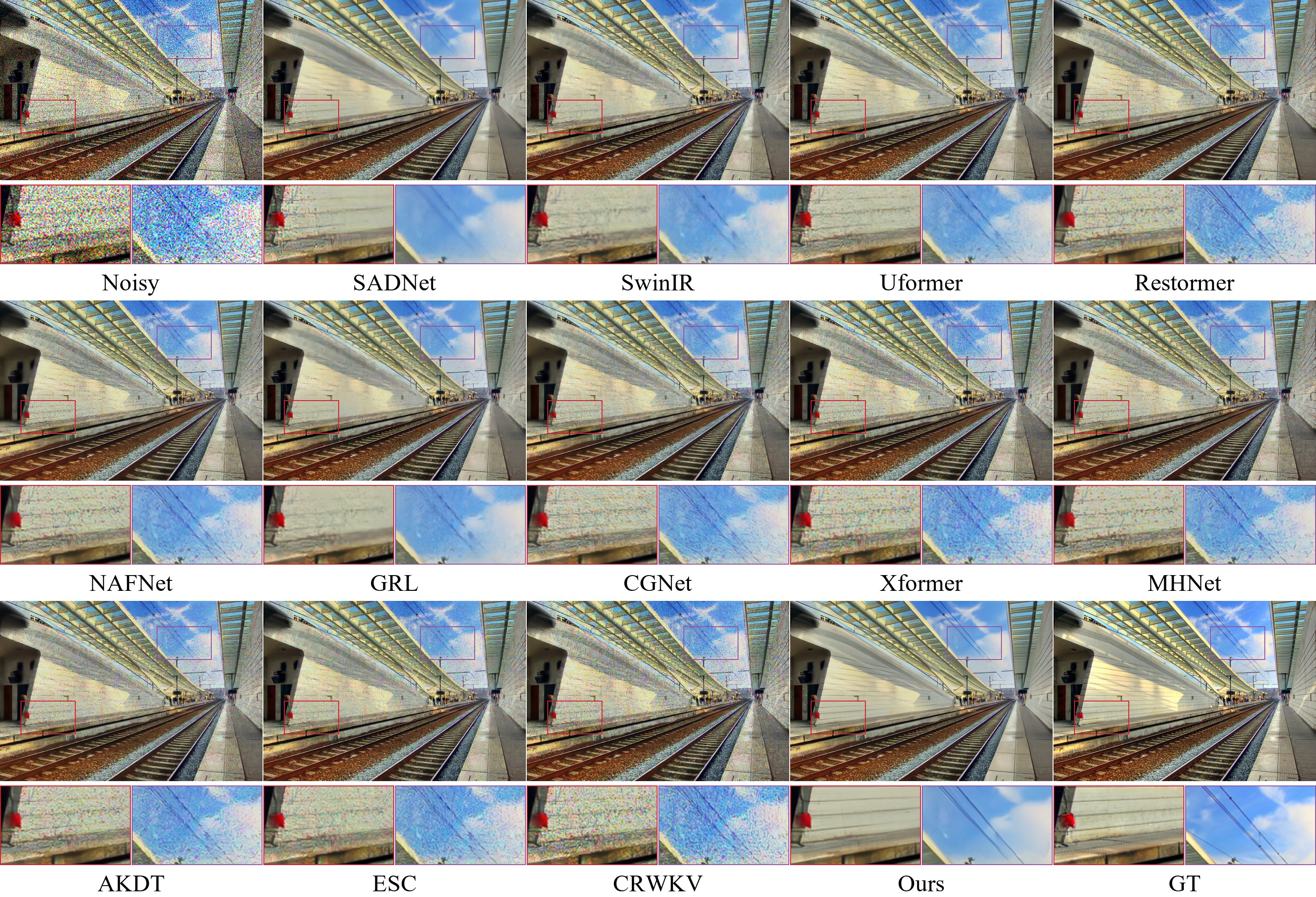}
\caption{Qualitative comparisons on Poisson noise of $\alpha=5$.}
\vspace{-15pt}
\label{fig:poisson(5)}
\end{figure*}

\begin{figure*}[t]
\centering
\includegraphics[scale=0.154]{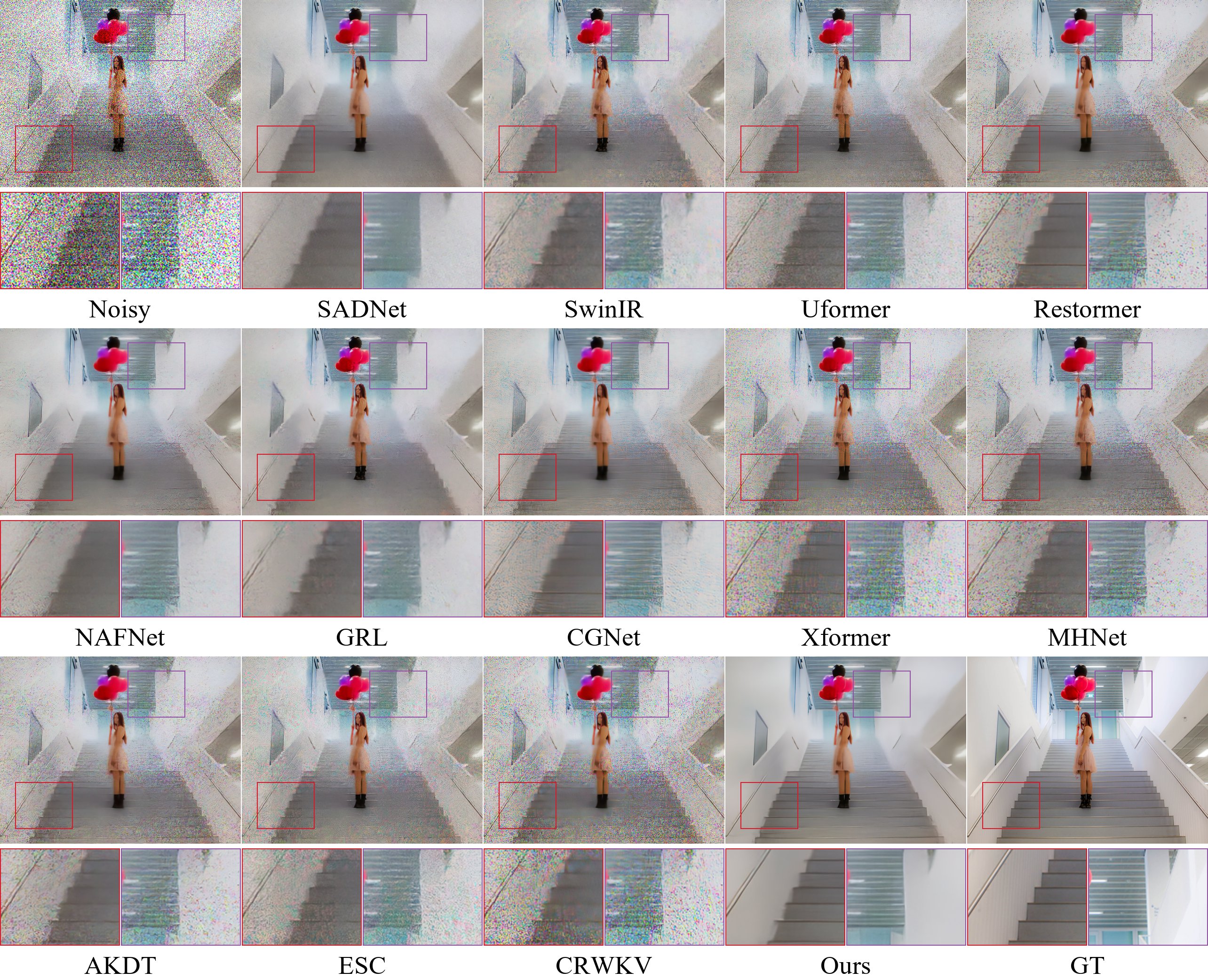}
\caption{Qualitative comparisons on variant 1 of Speckle noise.}
\vspace{-20pt}
\label{fig:variant1}
\end{figure*}

\begin{figure*}[t]
\centering
\includegraphics[scale=0.14]{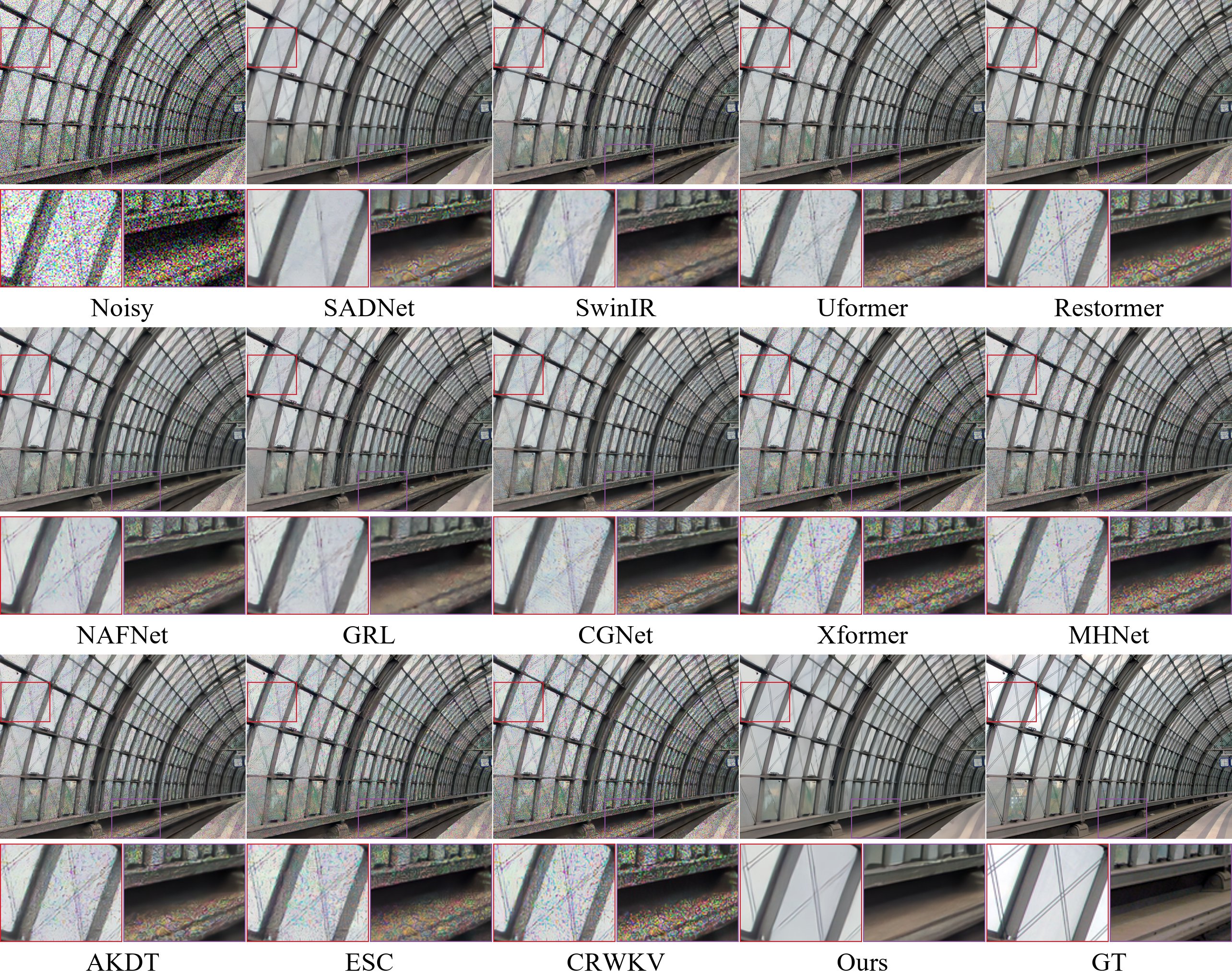}
\caption{Qualitative comparisons on variant 2 of Speckle noise.}
\vspace{-11pt}
\label{fig:variant2}
\end{figure*}

\begin{figure*}[t]
\centering
\includegraphics[scale=0.14]{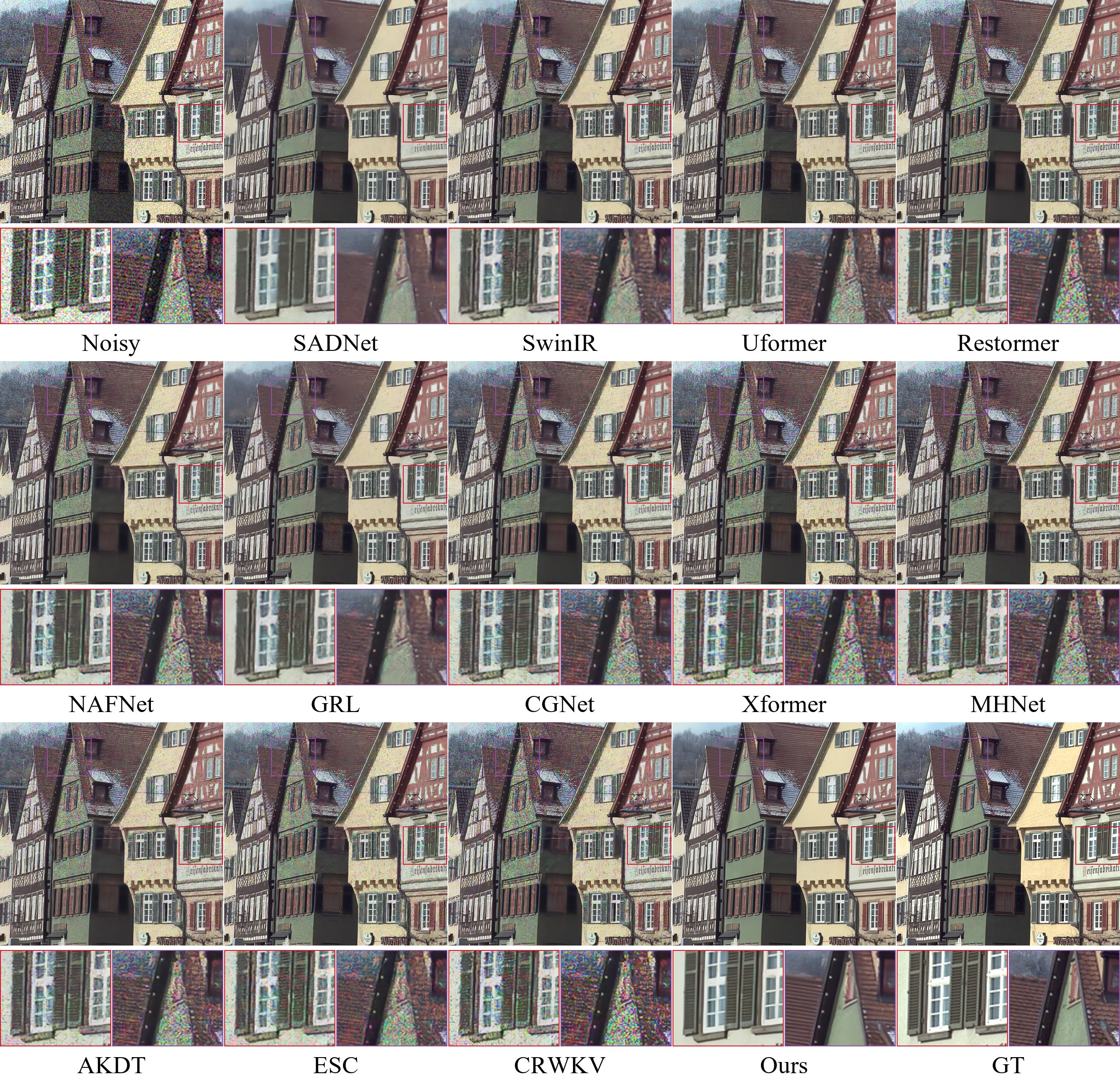}
\caption{Qualitative comparisons on variant 3 of Speckle noise.}
\vspace{-20pt}
\label{fig:variant3}
\end{figure*}

\begin{figure*}[t]
\centering
\includegraphics[scale=0.13]{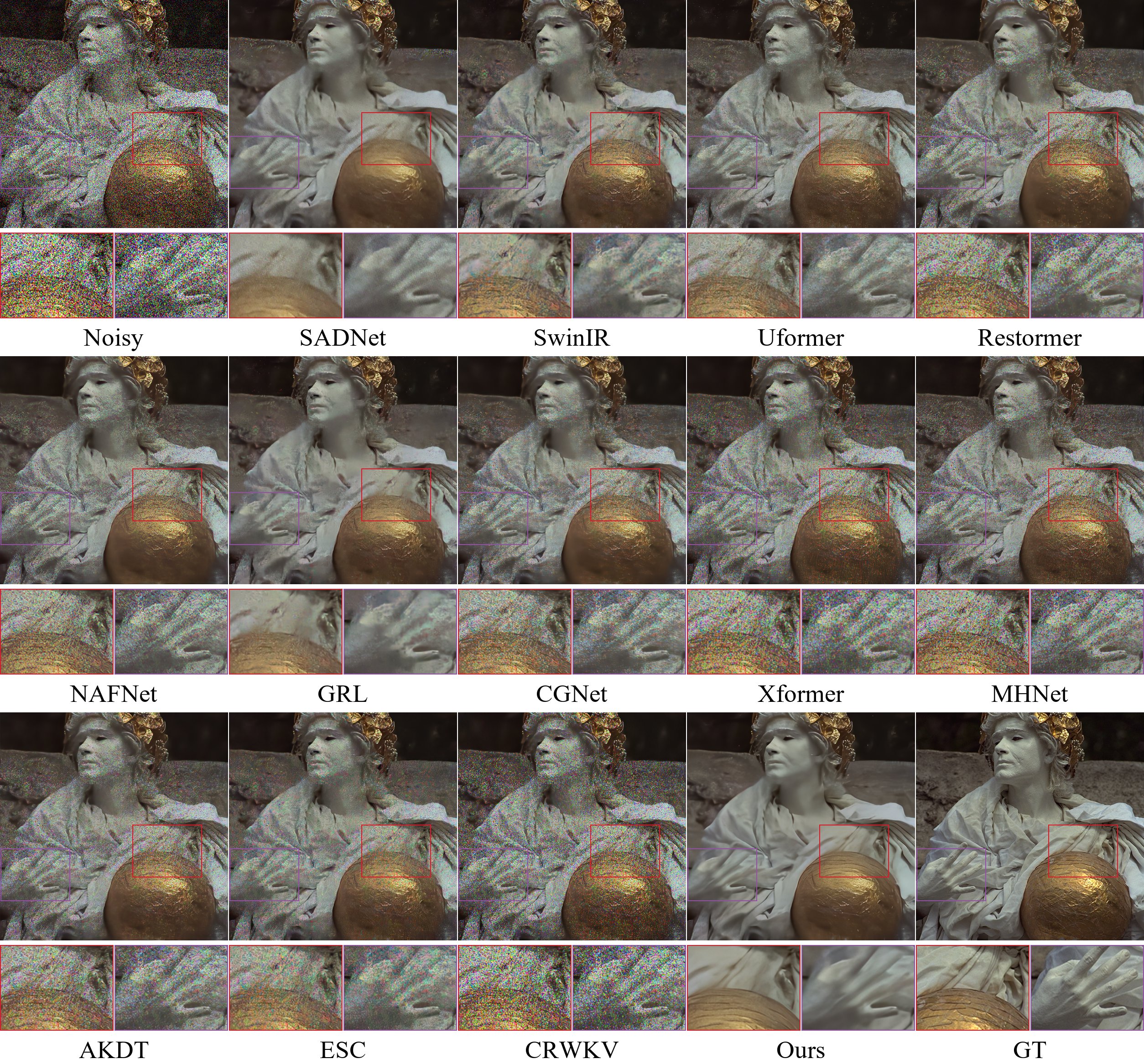}
\caption{Qualitative comparisons on variant 1 of Speckle noise*.}
\vspace{-15pt}
\label{fig:variant1*}
\end{figure*}

\begin{figure*}[t]
\centering
\includegraphics[scale=0.13]{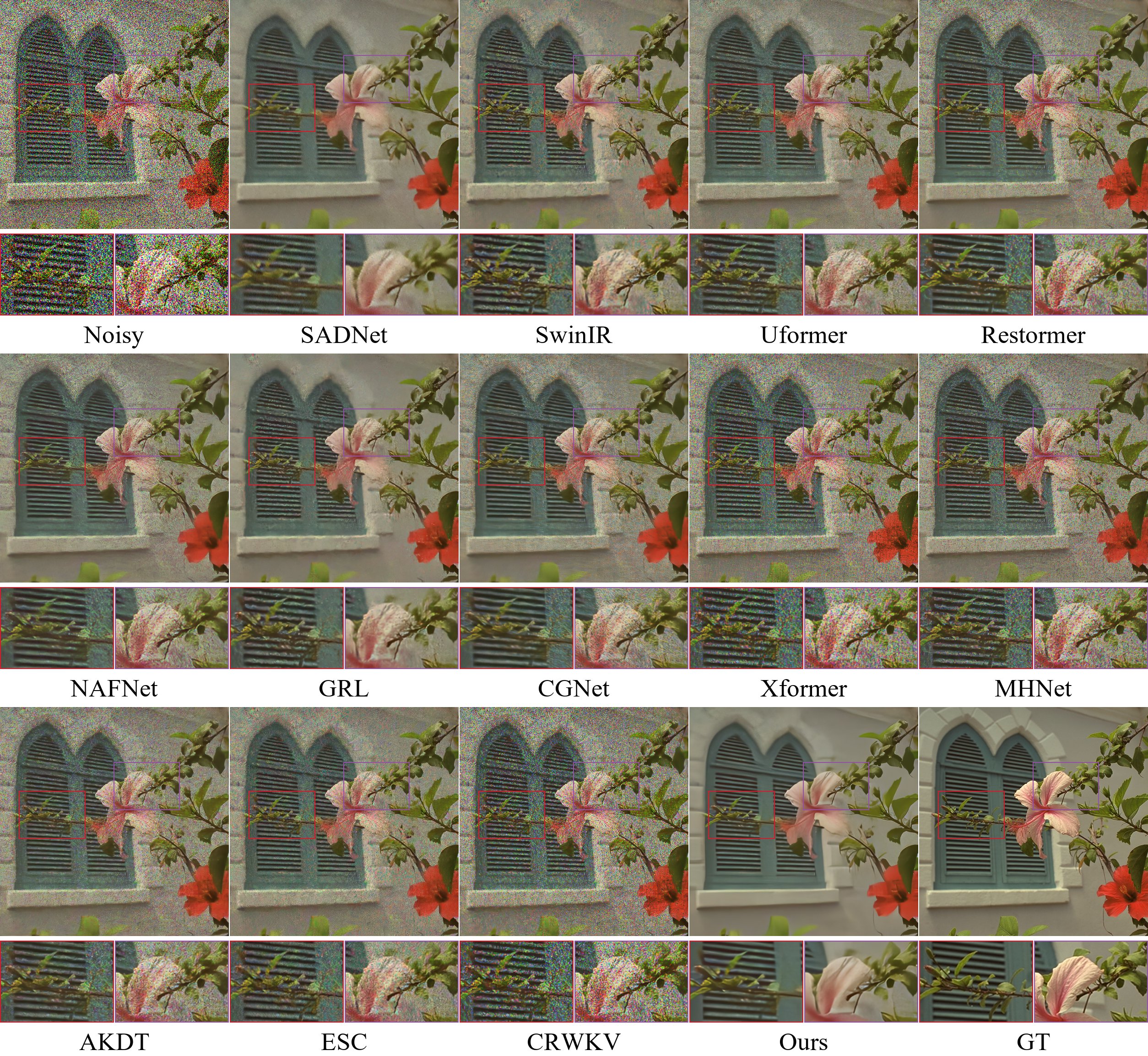}
\caption{Qualitative comparisons on variant 3 of Speckle noise*.}
\vspace{-15pt}
\label{fig:variant3*}
\end{figure*}

\begin{figure*}[t]
\centering
\includegraphics[scale=0.13]{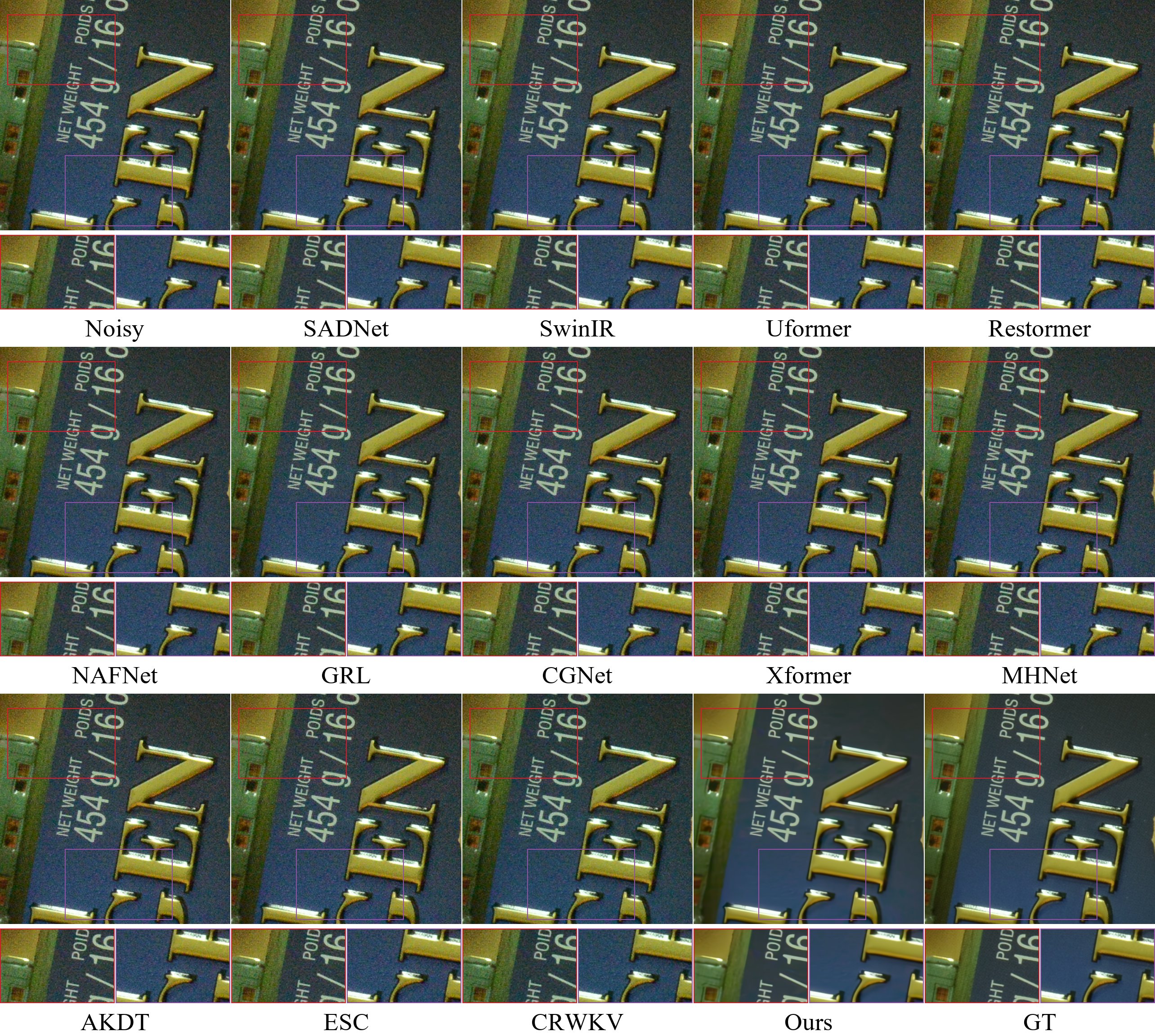}
\caption{Qualitative comparisons on CC benchmark.}
\vspace{-12pt}
\label{fig:cc1}
\end{figure*}

\begin{figure*}[t]
\centering
\includegraphics[scale=0.13]{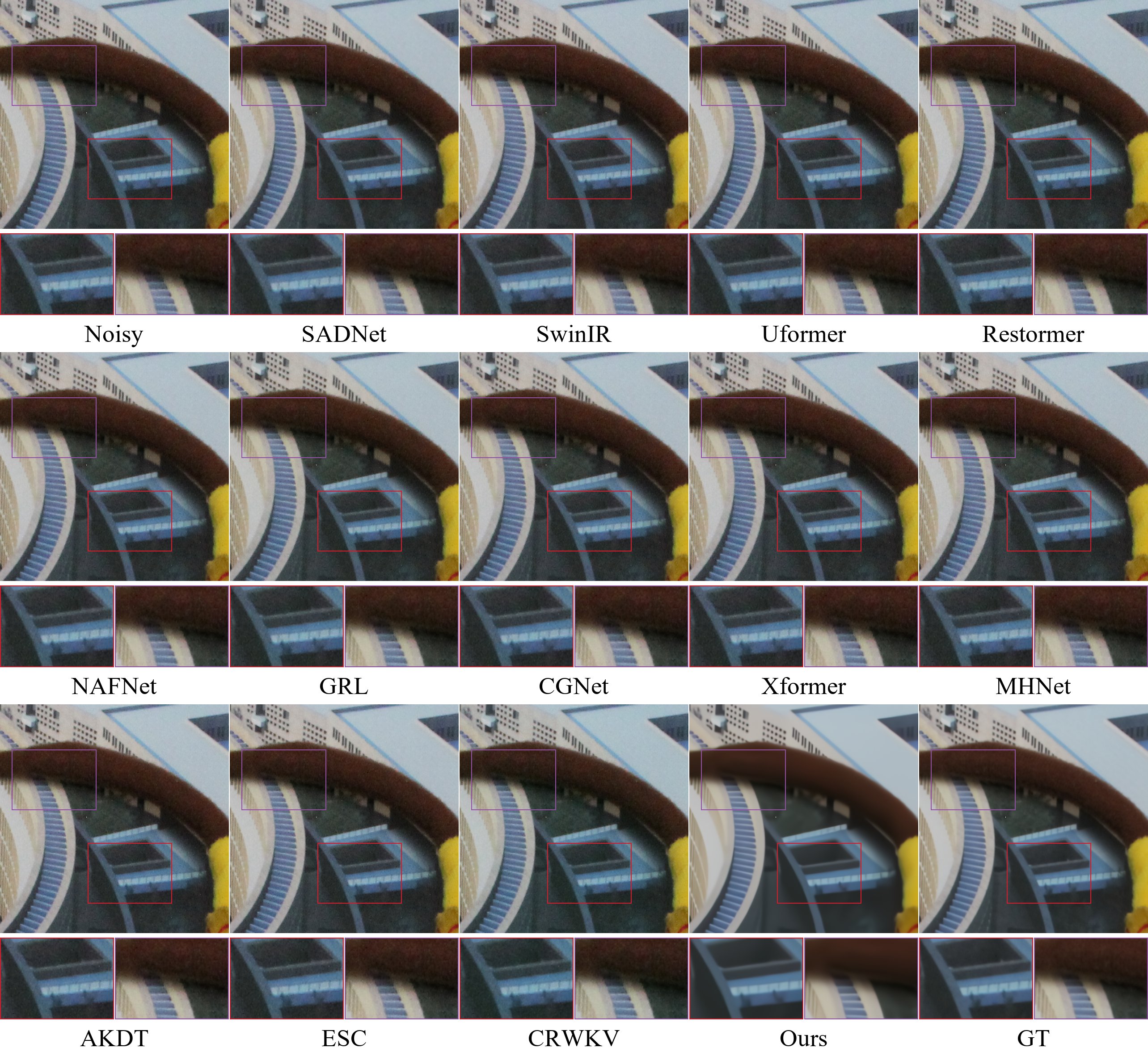}
\caption{Qualitative comparisons on HighISO benchmark.}
\vspace{-12pt}
\label{fig:highiso1}
\end{figure*}

\end{document}